\documentclass{article}

\usepackage{physics}

\usepackage{microtype}
\usepackage{graphicx}
\usepackage{subcaption}
\usepackage{booktabs} 

\usepackage{hyperref}

\usepackage[preprint]{icml2026}

\usepackage{amsmath}
\usepackage{amssymb}
\usepackage{mathtools}
\usepackage{amsthm}
\usepackage[inline]{enumitem}

\usepackage{multirow}
\usepackage[table]{xcolor}

\usepackage[capitalize,noabbrev]{cleveref}

\usepackage[textsize=tiny]{todonotes}

\icmltitlerunning{Stochastic Interpolants in Hilbert Spaces}

\crefname{appendix}{Appendix}{Appendices}
\Crefname{appendix}{Appendix}{Appendices}
\usepackage{afterpage}

\usepackage{multirow}
\usepackage{siunitx} 
\usepackage{thmtools}
\usepackage[most]{tcolorbox}

\declaretheorem[name=Theorem]{theorem}

\declaretheorem[sibling=theorem, name=Lemma]{lemma}
\declaretheorem[sibling=theorem, name=Definition, style=definition]{definition}
\declaretheorem[name=Remark, style=remark, numberwithin=theorem]{remark}
\declaretheorem[sibling=theorem, name=Hypothesis]{hypothesis}
\crefname{hypothesis}{hypothesis}{hypotheses}
\Crefname{hypothesis}{Hypothesis}{Hypotheses}

\newcommand*{\tran}{^{\mkern-1.5mu\mathsf{T}}}
\newcommand{\onorm}[1]{{\left\vert\kern-0.25ex\left\vert\kern-0.25ex\left\vert #1
\right\vert\kern-0.25ex\right\vert\kern-0.25ex\right\vert}}

\usepackage{ifplatform}

\newif\ifintcolorbox
\intcolorboxfalse

\newlength{\boxtweak}
\newlength{\remarktweak}
\newtcolorbox{theorembox}{
  breakable,
  colback=AliceBlue,
  colframe=SteelBlue,
  boxrule=0.75pt,
  sharp corners,
  before skip balanced=6pt,
  after skip balanced=6pt,
  before upper={\vspace*{\boxtweak}\intcolorboxtrue},
  after upper={\intcolorboxfalse},
  top=4pt,
  bottom=4pt,
  left=6pt,
  right=6pt,
  parbox=false,
}

\newtcolorbox{definitionbox}{
  breakable,
  colback=Honeydew,
  colframe=SeaGreen,
  boxrule=0.75pt,
  sharp corners,
  before skip balanced=6pt,
  after skip balanced=6pt,
  before upper={\vspace*{\boxtweak}\intcolorboxtrue},
  after upper={\intcolorboxfalse},
  top=4pt,
  bottom=4pt,
  left=6pt,
  right=6pt,
  parbox=false,
}

\newtcolorbox{remarkbox}{
  breakable,
  colback=LavenderBlush,
  boxrule=0pt,
  opacityframe=0,
  sharp corners,
  before skip balanced=6pt,
  after skip balanced=6pt,
  before upper={\vspace*{\remarktweak}\intcolorboxtrue},
  after upper={\intcolorboxfalse},
  top=4pt,
  bottom=4pt,
  left=6pt,
  right=6pt,
  parbox=false,
}

\usepackage{afterpage}

\usepackage{standalone}
\standaloneconfig{mode=buildnew}
\usepackage{etoolbox}

\makeatletter

\def\@LN@depthbox{%
  \ifdim\@tempdima = -1000pt
  \else
  \dp\@tempboxa=\@tempdima
  \nointerlineskip \kern-\@tempdima
  \fi
  \box\@tempboxa
}
\begin{document}
\twocolumn[
  \icmltitle{Stochastic Interpolants in Hilbert Spaces}



  \icmlsetsymbol{equal}{*}

  \begin{icmlauthorlist}
    \icmlauthor{James B. Yu}{yyy}
    \icmlauthor{RuiKang OuYang}{yyy}
    \icmlauthor{Julien Horwood}{yyy}
    \icmlauthor{José Miguel Hernández-Lobato}{yyy}
  \end{icmlauthorlist}

  \icmlaffiliation{yyy}{Department of Engineering, University of Cambridge, Cambridge, United Kingdom}

  \icmlcorrespondingauthor{James B. Yu}{jby21@cam.ac.uk}
  \icmlcorrespondingauthor{RuiKang OuYang}{ro352@cam.ac.uk}

  \icmlkeywords{Machine Learning, Diffusion Models, Neural Operators, Infinite Dimensions, Hilbert Spaces ICML}

  \vskip 0.3in
]
\newcommand{\tony}[1]{\todo{\textcolor{MidnightBlue}{[\textbf{Tony}: #1]}}}
\newcommand{\james}[1]{\todo{\textcolor{MintCream}{[\textbf{James}: #1]}}}



\printAffiliationsAndNotice{}  

\begin{abstract}
    Although diffusion models have successfully extended to function-valued data, stochastic interpolants -- which offer a flexible way to bridge arbitrary distributions -- remain limited to finite-dimensional settings. This work bridges this gap by establishing a rigorous framework for stochastic interpolants in infinite-dimensional Hilbert spaces. We provide comprehensive theoretical foundations, including proofs of well-posedness and explicit error bounds. We demonstrate the effectiveness of the proposed framework for conditional generation, focusing particularly on complex PDE-based benchmarks. By enabling generative bridges between arbitrary functional distributions, our approach achieves state-of-the-art results, offering a powerful, general-purpose tool for scientific discovery.  
\end{abstract}

\section{Introduction}

The success of modern generative models is exemplified by diffusion models  (DMs; \citealp{ho2020denoisingdiffusionprobabilisticmodels,karras2022elucidating,song2021scorebasedgenerativemodelingstochastic}), which have achieved state-of-the-art performance in diverse domains. In many applications, such as image generation \citep{song2021scorebasedgenerativemodelingstochastic}, video synthesis \citep{ho2022imagen}, weather modelling \citep{pathak2022fourcastnet}, PDEs \citep{huang2024diffusionpde}, and Bayesian inverse problems \citep{cardoso2023monte,wu2023practical,yao2025guideddiffusionsamplingfunction}, the data are inherently function-valued. In these domains, the data are commonly discretised as a \(D\)-dimensional vector and finite-dimensional DMs are applied, without addressing the performance of such an approach as \(D \to \infty\): \citet{chen2023improved,de2022convergence} show that model performance may deteriorate as discretisation becomes finer and  \citet{stuart2010inverse} shows that a naive finite-dimensional approach may introduce undesirable pathologies, particularly for Bayesian inverse problems. To address this, recent work has shown promising success in formulating DMs directly in infinite dimensions and discretising at the last step, i.e. when implementing on a computer, with remarkable success in PDE-based Bayesian inverse problems \citep{yao2025guideddiffusionsamplingfunction}.

Stochastic interpolants (SIs; \citealp{albergo2023stochasticinterpolantsunifyingframework,albergo2023stochastic}) are a powerful and more flexible alternative to DMs, capable of bridging between two arbitrary and possibly coupled distributions on a strict finite time horizon with explicit control over the interpolation path. This flexibility has led to state-of-the-art performance across domains such as image generation \citep{ma2024sit}, materials discovery \citep{hollmer2025open} and fluid simulation \citep{genuist2025autoregressive}, outperforming DMs. Despite these compelling advantages, a rigorous generalisation of SIs to infinite-dimensional spaces remains a crucial and unaddressed research problem.

Our work closes this gap by establishing a framework for SIs directly in infinite dimensions. By enabling bridges between arbitrary functional distributions, our framework offers a powerful class of functional generative models with more flexibility and stability than existing methods. We evaluate our method on challenging PDE-based forward and inverse problems \citep{yao2025guideddiffusionsamplingfunction,huang2024diffusionpde}, achieving competitive or state-of-the-art results.

\paragraph{Contributions} Our primary contributions are as follows:
\begin{enumerate*}[label=(\roman*)]
  \item we formulate Stochastic Interpolants directly in infinite-dimensional settings, with a main focus on conditional generation through a conditional bridge;
  \item we establish sufficient conditions under which the framework is well-posed and satisfies critical theoretical guarantees; and
  \item we demonstrate our framework's effectiveness for solving PDE-based forward and inverse problems, achieving competitive or state-of-the-art results.
\end{enumerate*}

\section{Preliminaries}

\subsection{Stochastic Differential Equations and Stochastic Interpolants in finite-dimensional spaces}
\paragraph{SDEs in $\mathbb{R}^d$.} We consider a \textit{Stochastic Differential Equation} (SDE) in Euclidean space with termination time $t=1$, which
defines a stochastic process $(X_t)_{t\in[0,1]}$ satisfying
{\begin{align}
    \dd{X_t}=f(t, {X_t})dt + g(t)\dd{{W_t}},\quad\text{where }X_0\sim p_0, \label{eq:forward-sde-Rd}
\end{align}}%
where $(W_t)_{t\in[0,1]}$ is a standard Wiener process, $f$ is referred to as the \emph{drift}, and $g$ as the \emph{diffusion}.
We refer to \eqref{eq:forward-sde-Rd} as the \emph{forward SDE}. Under mild regularity on $f$ and $g$, its time-reversal, \textit{a.k.a.} the \textit{backward SDE}, can be written as
\begin{multline}
    \dd{X_t} = \left[f(t, {X_t}) - g(t)^2\nabla\log p_t({X_t})\right] dt \\+ g(t)d\tilde{W_t},\quad\text{where }X_1\sim p_1,\label{eq:backward-sde-Rd}
\end{multline}
where $(\tilde{W_t})_{t\in[0,1]}$ is a time-reversed Wiener process, $p_t$ is the law of $X_t$, and $\nabla\log p_t(x)$ \footnote{For simplicity, we write $\nabla_x \log p_t(x)$ as $\nabla \log p_t(x)$ without further comment.} is the marginal score.

\paragraph{Stochastic Interpolants.}
A \emph{Stochastic Interpolant} \citep[SIs][]{albergo2023stochasticinterpolantsunifyingframework} bridges two arbitrary distributions $p_0$ and $p_1$ on \(\mathbb{R}^{d}\). It is a stochastic process $(x_t)_{t\in[0,1]}$ of the form
\begin{equation}
    x_t=\alpha(t)x_0+\beta(t)x_1+\gamma(t)z,\label{eq:SI-sample-formula}
\end{equation}
where $\alpha$, $\beta$, and $\gamma$ are chosen such that $x_0$ and $x_1$ are recovered at $t=0$ and $t=1$ respectively, $z \perp (x_0, x_1)$ is standard Gaussian noise, and $(x_0, x_1)$ are random variables sampled from any coupling of $p_0$ and $p_1$. The SIs \((x_{t})_{t \in [0, 1]}\) defines a family of distributions \((p_{t})_{t \in [0, 1]}\). 

For any diffusion coefficient \(g(t) \geq 0\), one can construct a separate stochastic process \((X_{t})_{t \in [0, 1]}\) governed by the following forward- and backward-SDEs, whose marginal distribution at time \(t\) coincides with that of \(x_{t}\). These SDEs reconstruct the marginal flow \(p_{t}\) but generally define fundamentally distinct stochastic processes:
{
\fontsize{9.4pt}{10.7pt}\selectfont
\begin{align}
    \dd{X_t} &= \left(b(t,X_t) + \frac{g(t)^2}{2}\nabla\log p_t({X_t})\right)\dd t + g(t)\dd W_t,\label{eq:forward-sde-si-Rd}\\
    \dd{X_t} &= \left(b(t,X_t) - \frac{g(t)^2}{2}\nabla\log p_t({X_t})\right)\dd t + g(t)\dd\tilde{W}_t\label{eq:backward-sde-si-Rd},
\end{align}}%
where \cref{eq:forward-sde-si-Rd} starts from $X_0\sim p_0$, \cref{eq:backward-sde-si-Rd} starts from $X_1\sim p_1$, and 
{
\fontsize{9.4pt}{10.7pt}\selectfont
\begin{align}
    b(t, x)&=\underbrace{\mathbb{E}[\dot\alpha_t x_0+\dot\beta_t x_1|x_t=x]}_{:=\varphi(t, x)}+\dot\gamma(t)\underbrace{\mathbb{E}[z|x_t=x]}_{:=\eta(t, x)}.
\end{align}}%
Notice that the denoiser $\eta_t$ and the score $\nabla\log p_t$ are related as $\eta_t(x)=-\gamma(t)\nabla\log p_t(x)$. Therefore, to obtain neural networks to enable traversing between $p_0$ and $p_1$, one can train a velocity network $\varphi_\theta(x,t)$ and a denoiser network $\eta_\phi(x, t)$ by minimizing the losses
\begin{align}
    \mathcal{L}_\varphi(\theta)&=\mathbb{E}[w_\varphi(t)\|\dot\alpha(t)x_0+\dot\beta(t)x_1-v_\theta(x_t, t)\|^2],\label{eq:velocity-loss-in-Rd}\\
    \mathcal{L}_\eta(\phi)&=\mathbb{E}[w_\eta(t)\|z-\eta_\phi(x_t, t)\|^2],\label{eq:denoiser-loss-in-Rd}
\end{align}
where $w_v$ and $w_\eta$ are positive weighting functions, and the expectation is taken over $(x_0, x_1)$, and $z$ with $x_t$ defined by \cref{eq:SI-sample-formula} and time \(t\) sampled uniformly on \([0,1]\).

\subsection{Gaussian Measures and Stochastic Processes in Hilbert Spaces}
We briefly review fundamental preliminaries for generalising stochastic interpolants to infinite dimensions. For a more comprehensive treatment, we refer the reader to \citet{bogachev1998gaussian,kukush2020gaussian,da2014stochastic}.

\paragraph{Hilbert Spaces.} 
Throughout, we consider a real, separable Hilbert space \(H\), which is infinite-dimensional, equipped with inner product, and has countable orthonormal basis, allowing us to view functions as vectors. Hence, we use the terms \textit{vector} and \textit{function} interchangeably.

\paragraph{Gaussian Measures in Hilbert Spaces.}
A random variable \(X \in H\) is distributed according to a \textit{Gaussian measure}  if, for all \(f \in H\), the inner product \(\ev{f, x}_{H} \in \mathbb{R}\) is a Gaussian random variable. We denote the law of \(X\) by \(\operatorname{N}(m, C)\), where $m\in H$ is the mean of \(X\) and $C: H \to H$ is a \textit{covariance operator} defined as a bounded, self-adjoint, positive-semidefinite, linear operator, satisfying:
\begin{align}
  \ev{Cf, g}_{H} &= \operatorname{Cov}\qty[\ev{f, X}_{H}, \ev{g, X}_{H}] \\
  &=\mathop{\mathbb{E}}\qty[\ev{f - m, X}_{H} \ev{g - m, X}_{H}],
\end{align}
for all \(f, g \in H\). \(C\) is called \textit{trace class}, if \( \operatorname{Tr}(C) \coloneqq \sum_{n=1}^{\infty} \ev{C e_{n}, e_{n}}_{H} = \sum_{n=1}^{\infty} \lambda_{n} < \infty\) for some eigensystem \(\qty{e_{n}, \lambda_{n}}_{n=1}^{\infty}\) of \(C\).

This condition is critical in infinite dimensions: for a Gaussian to be supported on \(H\), the expected squared norm of its samples, \(\mathbb{E}[\|X\|^2_H]=\norm{m}^{2}_{H} + \operatorname{Tr}(C)\), must be finite.
A Gaussian that does not satisfy this is said to be \textit{non-trace-class}, and has samples almost-surely unbounded in norm and hence not in \(H\). To ensure samples are well-defined, we focus only on  Gaussians with trace-class covariance.

\paragraph{Cameron-Martin Spaces.}

The \textit{Cameron-Martin space}, \(H_{C}\), is the image of \(H\) under \(C^{\frac{1}{2}}\). This is an (infinite-dimensional) subspace of \(H\) and itself a Hilbert space when equipped with the inner product \(\ev{f, g}_{H_{C}} \coloneqq \ev{C^{-\frac{1}{2}}f, C^{-\frac{1}{2}}g}_{H}\). If \(C\) is trace class, the operator \(C^{-\frac{1}{2}}\)  is unbounded  on \(H\) and \(H_{C}\) is a strict and dense subspace of \(H\). Since the eigenvalues of \(C\) are typically lowest for high-frequency modes, this condition means that elements of \(H_{C}\) are fundamentally ``smoother'' as they are constrained to have little energy in their high-frequency components.

When \(C\) is trace-class, samples from \(N(0, C)\) are almost surely not in \(H_{C}\). Intuitively, samples from \(N(0,C) \) are too ``rough'' to count as part of the smaller subspace of smoother functions \(H_{C}\).

\paragraph{Stochastic Processes in $H$.} For a trace-class covariance operator \(C\), a \textit{\(C\)-Wiener process} \(W_{t}\) is an \(H\)-valued (its samples are in \(H\)) stochastic process with continuous trajectories and independent increments such that \(W_{0} = 0\) and the law of \(W_{t} - W_{s}\) is \(\operatorname{N}(0, (t - s) C)\) for \(t \geq s \geq 0\).

We consider \(H\)-valued SDEs on the time domain \([0, 1]\) of the following form:
\begin{equation}
  \dd{X_{t}} = f(t, X_{t}) \dd{t} + g(t) \dd{W_{t}}, \quad\text{where } X_{0} \sim \mu_{0}.\label{eqn:hsde}
\end{equation}
Drift \(f\) takes values in \(H\), \(g\) is scalar-valued, and \(\mu_{0}\) is a measure on \(H\).
There are two types of solutions to \Cref{eqn:hsde}. Intuitively, a \emph{weak solution} specifies the law of the process \(X_t\) and allows the underlying probability space and driving noise to vary. A \emph{strong solution} constructs \(X_t\) pathwise and is defined with respect to a fixed Wiener process \(W_{t}\) and initial condition \(X_{0}\). While the notion of weak solutions is sufficient for the purposes of generating samples, the uniqueness of strong solutions is more useful for our statement on Wasserstein bounds (\Cref{thm:w2}), where we employ a coupling argument. We refer the reader to \Cref{app:weak-and-strong} for definitions of weak and strong solutions, and to \citet{da2014stochastic,scheutzow2013stochastic} for an in-depth treatment.

\subsection{Challenges in Extending SIs to Infinite Dimensions} \label{sec:challenges}
To extend SIs to infinite dimensions, one should overcome the following challenges:
\paragraph{Choice of Gaussian Noise.} The Gaussian noise must have a trace-class covariance operator to ensure samples are well-defined, which rules out isotropic noise. 

\paragraph{No Lebesgue Measure.} Typically, in finite dimensions, densities are taken with respect to the Lebesgue measure. However, the Lebesgue measure does not exist in infinite dimensions. Crucially, this makes the score and hence the typical governing equations for the forward and time-reverse stochastic processes ill-defined. One might consider defining densities with respect to some reference Gaussian measure. However, due to the time-varying noise schedule, this approach faces a crucial obstacle stemming from the Feldman-Hajek theorem: on a Hilbert space, Gaussian measures are either equivalent or mutually singular, and in infinite dimensions a non-trivial rescaling of a trace-class covariance operator results in mutually singular measures \citep{bogachev1998gaussian}. Consequently, the laws of the interpolant at different times are not absolutely continuous with respect to any single reference Gaussian measure.


\paragraph{Well-Posedness of SDEs.} In finite dimensions, adding scaled noise to interpolated data has a regularising effect, ensuring the corresponding SDE is well-posed. This guarantee is lost in infinite dimensions, where the regularising effect of Gaussian noise on arbitrary measures is often insufficient. This can result in a drift term that is unbounded and/or non-Lipschitz, violating the conditions ensuring the uniqueness or even existence of solutions.


\section{A Framework of Stochastic Interpolants in Infinite Dimensions}
This section develops a framework in infinite-dimensional Hilbert spaces to establish a stochastic bridge between two arbitrary distributions which might be coupled. 

A natural question is whether Stochastic Interpolants in function spaces can be obtained as limits of finite-dimensional constructions under increasingly fine discretizations. However, as discussed in Section 2, such an approach is fundamentally problematic: finite-dimensional SIs rely on densities with respect to Lebesgue measure and isotropic Gaussian noise, neither of which admit a meaningful infinite-dimensional analogue. As the discretisation dimension grows, the resulting measures need not converge to a well-defined limit, and the associated score functions may become ill-posed or resolution-dependent \citep{franzese2025generative}. By formulating SIs directly in infinite-dimensional Hilbert spaces, we avoid taking limits of objects that are not stable under increasingly fine discretisation.

Additionally, while the finite-dimensional SIs proposed by \citet{albergo2023stochasticinterpolantsunifyingframework,albergo2023stochastic} can match source and target marginals, marginal correctness is insufficient when the two variables are statistically coupled. In particular, a marginal bridge does not guarantee that samples generated from a fixed source follow the correct conditional distribution of the target. This is problematic in inverse problems, where conditioning on the observed input is essential. 

Our framework extends \citet{albergo2023stochasticinterpolantsunifyingframework,albergo2023stochastic} in two distinct ways: 1) we prove the well-posedness of Stochastic Interpolants in infinite dimensions, and 2) we construct an explicit conditional bridge between arbitrarily coupled source and target distributions. The latter is particularly crucial for conditional generation, such as Bayesian inverse problems,  where conditioning variables are functionals.

Let \(H\) be a real, separable Hilbert space equipped with the inner product \(\ev{\cdot, \cdot}_{H}\). We denote probability measures by $\mu$. Whenever a measure admits a density with respect to a reference measure, we denote this density by $p$ and use the notations interchangeably. We consider the problem of bridging a source distribution $\mu_0$ and a target distribution $\mu_1$ on $H$.
Our framework is defined as follows:
\begin{definition}\label{dfn:stochint}%
Let $x_0$ be distributed according to a source measure $\mu_0$, and $x_1$ be distributed according to a target measure $\mu_1$.
A \textit{Hilbert Space Stochastic Interpolant} is an \(H\)-valued stochastic process \(\qty(x_{t})_{t \in [0, 1]}\) of the form
\begin{equation}
  x_{t} = \alpha(t) x_{0} + \beta(t) x_{1} + \gamma(t)z,\quad\text{where}
\end{equation}%
\begin{enumerate}[label=(\roman*)]
    \item the pair $(x_0,x_1)$ is sampled from a coupling measure $\pi$ on $H\times H$, whose marginals admit the source and target measures, i.e. \(\mu_{0}(\dd{x_{0}}) = \pi(\dd{x_{0}}\times H)\) and \(\mu_{1}(\dd{x_{1}}) = \pi(H \times \dd{x_{1}})\);
    \item the random variable \(z\) is independently drawn from a Gaussian measure \(\operatorname{N}(0, C)\), with \(C : H\to H\) a positive-definite trace-class covariance operator;
    \item The non-negative functions $\alpha, \beta \in C^1([0,1])$ and $\gamma \in C([0,1]) \cap C^1((0,1))$ satisfy the boundary conditions $\alpha(0)=\beta(1)=1$, $\alpha(1)=\beta(0)=\gamma(0)=\gamma(1)=0$, with $\gamma(t)>0$ for $t \in (0,1)$.
\end{enumerate}
\end{definition}
For conciseness, we denote the time derivative of the process as \(\dot{x}_{t}\coloneqq \dot{\alpha}(t) x_{0} + \dot{\beta}(t) x_{1} + \dot{\gamma}(t) z\). The following results focus on the forward process; the reverse process follows by symmetry (\Cref{sec:backwards}).

\subsection{Conditional Generation through Conditional Bridges}

To formulate the conditional bridge, we require specific notation for conditional measures. For \(\mu_{0}\)-almost every \(x_{0}\), we denote by \(\mu_{t \mid 0}( \cdot\mid x_{0})\) the distribution of \(x_{t}\) conditional on \(x_{0}\). An analogous definition applies to \(\mu_{t \mid 1}( \cdot\mid x_{1})\).

We now construct a stochastic process called the \textit{Conditional Bridge} (CB). This is in contrast with guidance-based score interpolantion (e.g. classifier-free guidance, \citealp{ho2022classifier}), as the latter is ill-defined in infinite dimensions due to measure-theoretic pathologies discussed in \Cref{sec:challenges}.
Conditional on an initial state \(X_{0} = x_{0}\) drawn from \(\mu_{0}\), this process targets the conditional distribution \(\mu_{1 \mid 0}(\cdot\mid x_{0})\). The dynamics are governed by the following \textit{Conditional Bridge SDE} (CB-SDE):
\begin{equation}
\dd{X_{t}} = f(t, x_0, X_t)\dd t+ \sqrt{2\varepsilon} \dd{W_{t}}, \quad X_{0} = x_{0}, \label{eqn:cbsde}
\end{equation}
where \(W_t\) is a $C$-Wiener process on \(H\) and \(\varepsilon > 0\) is a diffusion parameter. The drift coefficient \(f : [0, 1] \times H \times H \to H\) is given by:
\begin{multline}
f(t, x_0, x_t) = \varphi(t, x_{0}, x_t) \\+ \qty(\dot{\gamma}(t) - \frac{\varepsilon}{\gamma(t)}) \eta(t, x_{0}, x_t),\label{eqn:2comp}
\end{multline}
where the \textit{conditional-bridge velocity} \(\varphi\) and the \textit{conditional-bridge denoiser} \(\eta\) are defined as follows:
\begin{align}
\varphi(t, x_{0}, x_{t}) &= \mathop{\mathbb{E}}\qty[\dot{\alpha}(t) x_{0} + \dot{\beta}(t) x_{1} \mid x_{0}, x_{t}],\label{eq:conditional-bridge-velocity} \\
\eta(t, x_{0}, x_{t}) &= \mathop{\mathbb{E}}\qty[ z \mid x_0, x_{t}].\label{eq:conditional-bridge-denoiser}
\end{align}


Assuming the CB-SDE has a unique solution in law on a subinterval \([0, \overline{t}] \subseteq [0, 1]\) for \(\mu_{0}\)-almost every initial condition,  we let \(\rho_{t \mid 0}(\cdot, x_{0})\) be the law of \(X_{t}\) solved by \cref{eqn:cbsde} at time \(t \in [0, \overline{t}]\), conditional on \(X_{0} = x_{0}\).
Our first result shows that the CB-SDE (\ref{eqn:cbsde}) correctly transports \(x_{0} \sim \mu_{0}\) to a conditional target distribution \(\mu_{1 \mid 0}(\cdot \mid x_{0})\). 

\begin{restatable}{theorem}{restatethmcbsde} \label{thm:cbsde}
Suppose that for any \(\overline{t} \in (0, 1)\), the CB-SDE (\ref{eqn:cbsde}) admits a unique (in law) solution on \([0, \overline{t}]\) for \(\mu_{0}\)-every initial condition \(x_{0}\) and the associated infinite-dimensional Fokker-Planck equation is uniquely solvable. Then, for any \(t \in (0, 1)\), the law of \(X_{t}\) conditional on \(X_{0} = x_{0}\) coincides with the conditional interpolant distribution \(\mu_{t\mid0}(\cdot \mid x_{0})\).
\end{restatable}
\begin{proof}
See \Cref{prf:thm:cbsde}, where we show that \(\mu_{t \mid 0}(\cdot \mid x_{0})\) is also a solution to the Fokker-Planck equation associated with the CB-SDE when conditioning on \(X_{0} = x_{0}\).
\end{proof}

Uniqueness of the CB-SDE is established in \cref{sec:eu} under explicit sufficient conditions, while uniqueness of the associated infinite-dimensional Fokker-Planck equation is assumed, which is standard in the literature \citep{bogachev2010uniquenesssolutionsfokkerplanckequations}. Since the interpolant \(x_{t}\) converges almost surely to \(x_{1}\) as \(t \uparrow 1\), the conditional target distribution \(\mu_{1 \mid 0}(\cdot\mid x_{0})\) is recovered as the limit of \(\mu_{t \mid 0}(\cdot\mid x_{0})\). In practice, simulating for \(t\) sufficiently close to \(1\) yields accurate approximations of the target distribution.

Unlike the finite-dimensional stochastic interpolant framework \citep{albergo2023stochasticinterpolantsunifyingframework}, we keep the diffusion factor \(\varepsilon\) constant and not time-varying. This is in line with regularity conditions concerning the uniqueness of solutions to Fokker-Planck equations \citep{bogachev2009fokker,bogachev2010uniquenesssolutionsfokkerplanckequations}, which are conditions for \Cref{thm:cbsde}. However, \Cref{app:fixed-epsilon} explains why we do not consider this a limitation.

\subsection{Existence and Uniqueness of Strong Solutions}\label{sec:eu}
While \Cref{thm:cbsde} only requires existence and uniqueness of solutions to the CB-SDE in law, we focus on strong solutions in order to facilitate later analysis of Wasserstein distances between generated samples and the true target distribution (see \Cref{thm:w2}). In particular, strong solutions allow us to couple SDEs driven by the same Wiener process.

Establishing existence and uniqueness of strong solutions to the conditional bridge SDE is substantially more delicate in infinite dimensions than in finite-dimensional settings. Although the driving Wiener process \(W_{t}\) is Gaussian, its covariance operator \(C\) is trace-class and injects stochasticity only along directions determined by the Cameron-Martin geometry. Hence, the noise does not regularise the drift uniformly in the ambient Hilbert space \(H\), and classical well-posedness results based on Lipschitz continuity in the \(H\)-norm do not apply. Instead, the natural notion of regularity is Lipschitz continuity with respect to the stronger Cameron-Martin norm induced by \(C\). 

In this section, we identify two concrete and practically relevant conditions under which this anisotropic regularity holds, and use this \(H_{C}\)-Lipschitz regularity to establish existence and uniqueness of strong solutions on compact time intervals.

To this end, we consider two complementary assumptions on the data distribution. The first setting is motivated by Bayesian forward and inverse problems, where the data admit a density with respect to a reference Gaussian measure supported on the Cameron–Martin space.

\begin{hypothesis} \label{hyp:bayes}%
Let \(H_{C} \coloneqq C^{\frac{1}{2}}H\) be the Cameron-Martin space of \(C\). We suppose the following conditions hold.
\begin{enumerate}[label=(\roman*)]
  \item \label{hyp1.1} The law \(\pi\) of data \((x_0,x_1)\), is supported on the product space \(H_{C}^{2} \coloneqq H_{C} \times H_{C} \) and has zero mean and a density \(p : H^{2}_{C} \to \mathbb{R}_{\geq 0}\) w.r.t. a \textit{prior} Gaussian measure \(\operatorname{N}(0, Q)\) on \(H^{2}_{C}\), where \(Q\) is a positive-definite trace-class covariance operator on \(H_{C}^{2}\)
  \item \label{hyp1.3} The negative log-density \(\Phi \coloneqq - \log p\) is twice differentiable and strongly convex. 
\end{enumerate}
\end{hypothesis}
The second setting dispenses with the density assumption and instead requires bounded support of the data in the Cameron–Martin norm. This formulation is particularly well suited to scenarios in which the data are subject to geometric constraints, for instance when they lie on a low-dimensional manifold embedded in an infinite-dimensional function space \citep{pidstrigach2022score,de2022convergence}.

These two cases are stated in the following proposition:
\begin{restatable}{proposition}{restatelemmanifold}\label{lem:lipchitz-conditions}
The map \(x \mapsto f(t, x_{0}, x)\) (\cref{eqn:cbsde}) is Lipschitz continuous with respect to the Cameron-Martin norm, $\norm{\cdot}_{H_C}$. That is, for any \(t \in (0, 1)\) and \(x_{0}, x,y \in H\),
\begin{equation}
  \norm{f(t, x_{0}, x) - f(t, x_{0}, y)}_{H_{C}} \leq L(t) \norm{x - y}_{H_{C}},\label{eq:lipschitz-of-f}
\end{equation}
where the Lipschitz constant \(L(t)\) is finite on any compact subinterval of \((0, 1)\), provided that at least one of the following conditions holds:
\begin{enumerate}[label=(\roman*)]
    \item \Cref{hyp:bayes} is satisfied; or \label{setting:bayes}
    \item the law \(\mu_{1}\) of the target data \(x_{1}\) is supported on a bounded subset of \(H_{C}\) with respect to the \(H_{C}\)-norm.\label{setting:manifold}
\end{enumerate}
\end{restatable}
\begin{proof}
We treat each of \ref{setting:bayes} and \ref{setting:manifold} in \Cref{prf:prp:bayes,prf:prp:manifold} respectively.
\end{proof}

\begin{remark}\label{rm:hc}
Both settings \ref{setting:bayes} and \ref{setting:manifold} rely on the assumption that the target data \(x_{1}\) lies in the Cameron-Martin space \(H_{C}\). In SIs, different candidate targets \(x_{1}', x_{1}''\) induce Gaussian likelihoods for \(x_{t}\) of the form \(\operatorname{N}(\alpha(t) x_{0} + \beta(t) x_{1}, \gamma^{2}(t) C)\). By the Cameron-Martin and Feldman-Hajek theorems \citep{bogachev1998gaussian,stuart2010inverse}, these Gaussian measures are equivalent if and only if \(x_{1}' - x_{1}'' \in H_{C}\); otherwise they are mutually singular. The restriction \(x_{1} \in H_{C}\) thus guarantees that the posterior distribution of \(x_{1}\) admits a well-defined density w.r.t. a common reference measure. This property is essential for establishing Lipschitz continuity of the drift defined via conditional expectations
.
\end{remark}

\begin{remark}
The Lipschitz constant \(L(t)\) diverges as \(t \to 1\) for any choice of \(\gamma\), reflecting the singular drift of diffusion bridge processes at the terminal time \citep{li2016generalised,pidstrigach2023infinite}. This behaviour is intrinsic to finite-time diffusion bridges and is not specific to our construction. Accordingly, we establish existence and uniqueness of strong solutions on any compact subinterval of \([0,1)\).
\end{remark}

Building on  the  Lipschitz-continuity properties for the drift coefficient, established in \cref{lem:lipchitz-conditions}, we can now address the existence and uniqueness of solutions to the CB-SDE (\ref{eqn:cbsde}). We first address the existence.
\begin{restatable}[Existence]{theorem}{restatethmexist} \label{thm:exist}
If there exists a \(\overline{t} \in (0, 1]\) such that%
\begin{enumerate}[label=(\roman*)]
    \item for each \(t \in (0, \overline{t})\) and \(\mu_{0}\)-almost every \(x_{0}\), the mapping \(x \mapsto f(t, x_{0}, x)\) is \(H_C\)-Lipschitz continuous; and
    \item the Lipschitz constant \(L(t)\) is continuous on \((0, \overline{t}]\) and its supremum, \(\sup_{t \in (0, \overline{t}]} L(t) \), is finite.
\end{enumerate} 
Then there exists a strong solution to the CB-SDE (\ref{eqn:cbsde}) on the time interval \([0, \overline{t}]\).
\end{restatable}
\begin{proof}[Proof]
See \Cref{prf:thm:exist}.
\end{proof}


Proving strong uniqueness in infinite dimensions is delicate due to the anisotropic action of the noise.  We provide one sufficient condition for strong uniqueness that is compatible with common diagonal covariance priors used in function-space modelling \citep{stuart2010inverse,da2014stochastic,dashti2017bayesian}.


\begin{restatable}[Uniqueness]{theorem}{restatethmuniq}\label{thm:uniq}
Assume the Lipschitz conditions of \Cref{thm:exist}.  If the target distribution \(\mu_1\) factorises along an eigenbasis of the covariance operator \(C\),  i.e. the coefficients \(\langle x_1, e_n\rangle\) are mutually independent, then the CB-SDE \eqref{eqn:cbsde} admits a unique strong solution on \([0,\overline t]\).
\end{restatable}
\begin{proof}
See \Cref{prf:thm:uniq}.
\end{proof}


\section{Learning Conditional Bridges in Infinite Dimensions}
With well-posedness established, we now describe how to parameterize and learn the CB-SDE drift. Similarly to the finite-dimensional cases \citep[Section 2.4]{albergo2023stochasticinterpolantsunifyingframework}, we propose learning the \textit{velocity} \(\varphi\) and \textit{denoiser} \(\eta\) as two components forming the drift \(f(t, x_{0}, x)\) in \cref{eqn:2comp}, by optimizing losses generalized from  $\mathbb{R}^d$ (\cref{eq:velocity-loss-in-Rd,eq:denoiser-loss-in-Rd}) to $H$:
{
\fontsize{9.4pt}{10.7pt}\selectfont
\begin{align}
    \mathcal{L}_\varphi(\theta)&=\mathbb{E}[w_\varphi(t)\|\dot\alpha(t)x_0+\dot\beta(t)x_1-\varphi_\theta(x_t, t)\|^2_H],\label{eq:velocity-loss-in-hilbert}\\
    \mathcal{L}_\eta(\phi)&=\mathbb{E}[w_\eta(t)\|z-\eta_\phi(x_t, t)\|^2_H],\label{eq:denoiser-loss-in-hilbert}
\end{align}}%
where, the expectation is taken over $t$, $(x_0,x_1)$, $z$, and $w_\varphi$ and $w_\eta$ are positive weighting functions.

\Cref{prp:loss} in \Cref{prf:prp:loss} establishes that minimising the above losses (\cref{eq:velocity-loss-in-hilbert,eq:denoiser-loss-in-hilbert}) is equivalent to regressing the conditional-bridge velocity (\cref{eq:conditional-bridge-velocity}) and the conditional-bridge denoiser (\cref{eq:conditional-bridge-denoiser}), as they differ only by a finite constant. Although this equivalence is standard in finite dimensions, its extension to infinite-dimensional spaces is non-trivial, as it requires establishing well-defined \(H\)-valued conditional expectations and projection properties in the absence of a Lebesgue reference measure and under trace-class (anisotropic) Gaussian noise.

\subsection{Regularising Time Change}\label{sec:tc}
For arbitrary \(\varepsilon > 0\), the coefficient \(c(t) \coloneqq \dot{\gamma}(t) - {\varepsilon}/{\gamma(t)}\) on the denoiser \(\eta\) \cref{eqn:2comp} could diverge as \(t\) approaches the endpoints \(0\) and \(1\). This means that if \(\varphi\) and \(\eta\) are approximated with finite error, there is no finite guarantee on the approximation error of the overall drift, due to the  diverging coefficient \(c(t)\). To mitigate this, we introduce a change-of-time which cancels out the singularity introduced by \(c(t)\) and ensures that the coefficient on the denoiser is finite on the entire interval \([0, 1]\). The integrability of \(1/{\gamma(t)}\) on \([0, 1]\) is a crucial  condition allowing for such a construction. We state this in the following result.

\begin{restatable}{lemma}{restatelemtc}\label{lem:tc}
Let the coefficient \(c(t) \coloneqq \dot{\gamma}(t) - {\varepsilon}/{\gamma(t)}\). Suppose the improper integral \(\int_{0}^{1} {1}/{\gamma(t)} \dd{t}\) is finite and the product \(\dot{\gamma}(t) \gamma(t) \) has a (unique) continuous extension on \([0, 1]\). Then, there exists a strictly increasing, bijective, continuously differentiable time change \(\theta(t) : [0,1] \leftrightarrow [0, 1]\) such that the time-transformed coefficient
\begin{equation}
  \hat{c}(t) \coloneqq c(\theta(t))\dot{\theta}(t) = \qty(\dot{\gamma}(\theta(t)) - \frac{\varepsilon}{\gamma(\theta(t))})\dot{\theta}(t), \label{eqn:chat}
\end{equation}
defined for \(t \in (0, 1)\), has a continuous extension on the compact interval \([0, 1]\).
\end{restatable}
\begin{proof}
See \Cref{prf:lem:tc}.
\end{proof}
The time-changed stochastic process \(\hat{X}_{t} \coloneqq X_{\theta(t)}\) satisfies the following SDE, which we call the \emph{Time-changed Conditional Bridge SDE} (TC-CB-SDE):
\begin{equation}
\dd{\hat{X}_{t}} \coloneqq f(\theta(t), X_{0}, \hat{X}_{t}) \dot{\theta}(t)\dd t + \sqrt{2\varepsilon \dot{\theta}(t)} \dd{ \hat{W}_{t}}, \label{eqn:tccbsde}
\end{equation}
where \(\hat{W}_{t}\)  is a \(C\)-Wiener process. Since \(\theta\) is strictly increasing and bijective on \([0, 1]\), the TC-CB-SDE has a unique strong solution on \([0, \theta^{-1}(\overline{t})]\) as long as \(X_{t}\) has a unique solution on \([0, \overline{t}]\). Intuitively, the reparameterisation slows down time near the original singularities, causing the time-changed process to spend more ``simulation time'' at the endpoints and hence regularise the drift. The benefits of the TC-CB-SDE are two-fold. First, it improves numerical stability by preventing training errors from being amplified by the singular drift. Second, it is crucial for deriving a finite Wasserstein error bound, which we present next.




\subsection{Error bounds with Wasserstein-2 Distance}
Having established conditions under which our framework is well-posed, we now quantify how approximation errors in the learned velocity and denoiser fields affect the distribution of generated samples. In particular, we show that the Wasserstein--2 distance between the true conditional law $\mu_{t\mid 0}(\cdot \mid x_0)$ and its approximation $\tilde{\mu}_{t\mid 0}(\cdot \mid x_0)$ is controlled directly by the mean-squared errors of the learned velocity and denoiser. Consequently, as these fields are learned increasingly accurately, the induced generative distribution converges to the target distribution.

The Wasserstein-2 distance ($\mathcal{W}_2$) measures the distance between two measures and is defined as
\begin{equation}
\mathcal{W}_2^2(\pi_0,\pi_1)=\inf_{\pi_{\times}} \int_{H^{2}} \norm{x - y}^{2}_{H} \pi_{\times}(\dd{(x, y)}),
\end{equation}
where the infimum is taken over the space of joint measures \(\pi_{\times}\), which marginalise on \(\pi_{0}\) and \(\pi_{1}\), on \(H^{2}\).

Our bound exploits the regularity in the drift coefficient introduced by the change-in-time, which ensures that a finite approximation error in \(\widetilde{\varphi}\) and \(\widetilde{\eta}\) translates to a finite approximation error in the overall drift.

\begin{restatable}{theorem}{restatetheoremw}\label{thm:w2}
Let \(\widetilde{\varphi}\) and \(\widetilde{\eta}\) be the approximations of \(\varphi\) and \(\eta\) respectively, and $\widetilde{X}_t$ the solution of the corresponding CB-SDE starting from a given $X_0=x_0$. Suppose that
\begin{enumerate}[label=(\roman*)]
    \item \(\gamma(t)\) and \(c(t)\) satisfy the conditions in \Cref{lem:tc};
    \item for all $t \in [0, 1], x_{0} \in H$, \(\widetilde{\varphi}(t, x_{0}, x)\) and \(\widetilde{\eta}(t, x_{0}, x)\) are Lipschitz continuous w.r.t. $x$ in \(H\)-norm with time-dependent Lipschitz constant \({L}(t)\);
    \item The CB-SDE has a unique strong solution \(X_{t}\) on \([0, \overline{t}] \subseteq [0, 1]\) with sufficient conditions stated in \Cref{lem:lipchitz-conditions}.
\end{enumerate}
Then the $\mathcal{W}_2$ distance between $\mu_{t|0}(\cdot \mid x_0)$ and the corresponding measure $\widetilde{\mu}_{t|0}(\cdot \mid x_0)$, for any $t\in[0,\bar{t}]$, is upper bounded as follows:
\begin{multline}
    \mathcal{W}_2^2(\widetilde{\mu}_{t|0}(\cdot \mid x_{0}), \mu_{t|0}(\cdot \mid x_{0}))
  \leq\\ e^{C(t)} \int_{0}^{t} (A(s) +c^{2}(s)B(s))\dot\theta(\theta^{-1}(s))\dd s,\label{eqn:w2}
\end{multline}
where
\begin{align}
    C(t) &= \max_{s\in[0,t]} 2L(s)\dot\theta(s)(1+|c(\theta(s))|)<\infty\\
    A(s) &= \mathbb{E}\left[\norm{\tilde\varphi(s, x_0, x_s)-\varphi(s, x_0, x_s)}^2_H\right]\\
    B(s) &= \mathbb{E}\left[\norm{\tilde\eta(s, x_0, x_s)-\eta(s, x_0, x_s)}^2_H\right]
\end{align}
\end{restatable}
\begin{proof} The full proof is presented in \Cref{prf:thm:w2}. 
\end{proof}

This bound requires Lipschitz continuity of the \textit{learned} models, a condition satisfied by our neural operator architectures \citep{li2020fourier,rahman2022u}. Lipschitz continuity guarantees the \textit{approximated} SDE is well-posed on the entire \([0, 1]\) time domain, as its time-changed drift is uniformly Lipschitz. We can therefore simulate over the full interval, even though the true drift is singular at \(t=1\).

\section{Related Work}
\paragraph{DMs in Infinite Dimensions.} Recent work in generalising DMs to infinite dimensions has shown strong performance in functional settings \citep{franzese2025generative}: \citet{lim2023score} establish viability of DMs with general forward SDEs by deriving a time-reversal formula in infinite dimensions, while \citep{pidstrigach2023infinite} demonstrate that re-writing score as a conditional expectation using a generalisation of Tweedie's formula \citep{efron2011tweedie,yao2025guideddiffusionsamplingfunction} avoids needing a time-uniform reference measure in infinite dimensions. \citet{yao2025guideddiffusionsamplingfunction} employ a similar approach, proving Tweedie's formula more generally in infinite dimensions, and demonstrate state-of-the-art results on Bayesian forward and inverse problems.

\paragraph{Stochastic bridges.} Stochastic bridges have also been considered in infinite dimensions: \citep{yang2024infinite,baker2024conditioning} derive a time-reversal formula to reverse a diffusion bridge, but only bridge between predefined start and end points. \citet{park2024stochastic} bridge between arbitrary source and target distributions via stochastic optimal control, but consider only time-independent stochastic dynamics and require an on-path training objective which is computationally demanding and difficult to scale to larger datasets. 

\paragraph{SIs with Coupled Data.}
\citet{albergo2023stochastic} show SIs trained in a setting where source and target data are coupled can greatly simplify sampling trajectories. This exploits the flexibility of SIs as a stochastic bridge between arbitrary distributions, incorporating conditional information via paired data rather than the continuous embeddings used in DMs. However, their mathematical framework does not explicitly guarantee a \textit{conditional bridge} in which samples from the target distribution can be explicitly taken conditional on a starting sample from a coupled source distribution.

\paragraph{Forward and Inverse Problems.} DMs provide strong priors for Bayesian inverse problems in multiple domains \citep{cardoso2023monte,wu2023practical}, with an extension to infinite dimensions considered by \citet{baldassari2023conditional,baldassari2024taming}. \citet{yao2025guideddiffusionsamplingfunction} train an unconditional model and use a guidance mechanism during inference for forward and inverse PDE-based problems, demonstrating state-of-the-art performance compared with finite-dimensional DMs \citep{yao2025guideddiffusionsamplingfunction,huang2024diffusionpde}. We choose this as a primary baseline for comparison due to the challenging nature of this task, training a conditional model, which avoids needing to tune sensitive guidance parameters for inference.

\section{Experiments}
In this section, we evaluate our framework on functional datasets. We first conduct ablations on a dataset of 1D functions, and then present our main results on 2D PDE-based inference tasks. \Cref{sec:2instantiation} details the concrete instantiation of our framework employed in these experiments.

Although our framework is formulated in function space. Norms are approximated in practice by mean-squared error over fixed, uniformly spaced sensor locations, which is standard in operator learning and PDE inference \citep{yao2025guideddiffusionsamplingfunction,pidstrigach2025infinitedimensionaldiffusionmodels}. Full algorithmic details are given in \Cref{alg:training,alg:sampling}.

\subsection{1D PDE-based Inference}
\begin{figure*}[t]
  \centering
  \begin{subfigure}[t]{0.49\textwidth}
    \centering
    \includegraphics[width=\linewidth]{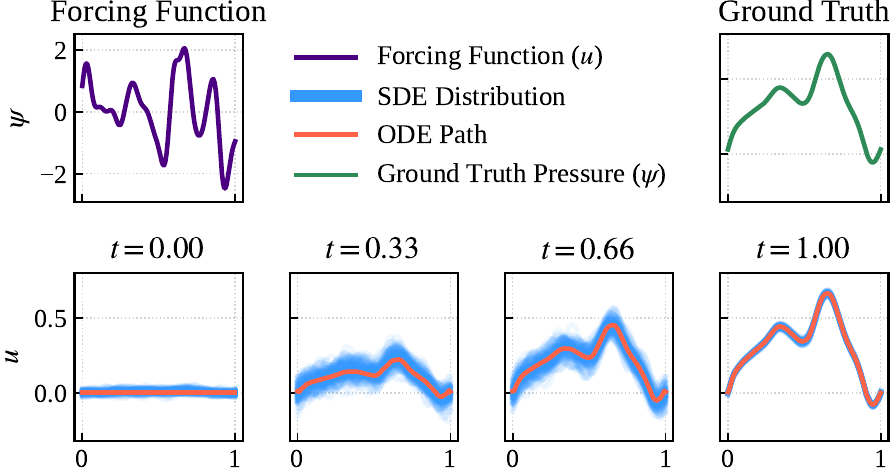}
    \caption{Forward}\label{fig:evolution:a}
  \end{subfigure}
  \hfill 
  \begin{subfigure}[t]{0.49\textwidth}
    \centering
    \includegraphics[width=\linewidth]{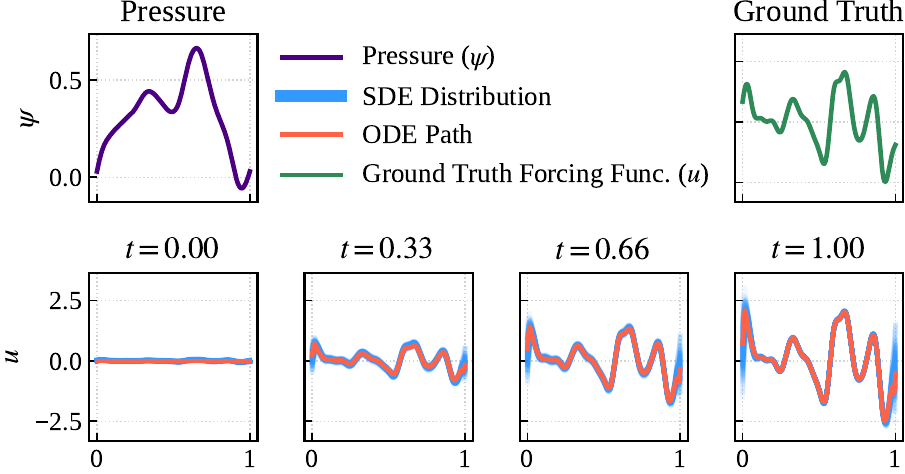}
    \caption{Inverse}\label{fig:evolution:b}
  \end{subfigure}
  \caption{Evolution of predictions for forward (a) and inverse (b) tasks on a randomly chosen example in the test set. For SDE inference, we generate 150 predictions to visualise the distribution of samples.} 
  \label{fig:evolution}
\end{figure*}
We employ a dataset of functions on a 1D unit line to serve as a computationally cheap experiment for the purposes of ablations. Specifically, we use this to demonstrate that simulating the time-changed analogues of the interpolant processes is critical for numerical stability. 
\paragraph{Dataset.}
We follow \citet{ingebrand2025basis} and consider a non-linear variant of Darcy's equation in 1D. We use 9000/1000 train/evaluation samples of paired functions, discretised on 128 evenly-spaced points on \([0, 1]\).

\paragraph{Change-of-Time Ablation.}
\Cref{tbl:cot} shows that during inference, using samples generated with change-of-time schedules satisfying \Cref{lem:tc,lem:thetaconditions} leads to a tremendous improvement in prediction accuracy. This suggests that introducing a change-of-time is critical in an infinite-dimensional setting to alleviate numerical issues arising from singularities at the endpoints of the time domain. Throughout, we measure performance by relative \(L^{2}\) error. 

Performance is best for schedules such as \(\dot{\theta}(t) \propto t(1-t)\) that \textit{minimally} slow down, in the sense that their derivatives decay linearly at critical endpoints (\Cref{lem:thetaconditions} requires the derivative to decay \textit{at least} linearly). Inferior performance on schedules such as \(\dot{\theta}(t) \propto \exp((t(1-t))^{-1})\) whose derivatives decay faster than any polynomial suggests excessive slowing down of time near the boundaries is not beneficial and that a change-of-time should act merely as a device for numerical stability.

\begin{table}[t]
    \centering
    \sisetup{detect-weight, mode=text}
    \renewcommand{\arraystretch}{1.25}
    \caption{Comparison of different time-change schedules. Performance is measured in percentage (\%) relative \(L^{2}\)-error. Best values for ODE and SDE are marked in \textbf{bold} respectively}
    \resizebox{\linewidth}{!}{
        \begin{tabular}{llS[table-format=4.1] S[table-format=4.1]}
            \toprule
            & \textbf{Schedule} & \textbf{Forward (\%)} & \textbf{Inverse (\%)} \\
            \midrule
            \textbf{ODE} (\(\varepsilon = 0\)) 
            & None (\(\dot{\theta}=1\)) & 9433.4 & 2500.5 \\
            & \(\dot{\theta}(t) \propto t(1-t)\) & \bfseries 3.8 & \bfseries  17.4 \\
            & \(\dot{\theta}(t) \propto \exp((t(1-t))^{-1})\) & 57.4 & 44.3 \\
            \midrule
            \textbf{SDE} (\(\varepsilon=b/2\)) 
            & None (\(\dot{\theta}=1\)) & 1395.8 & 266.3 \\
            & \(\dot{\theta}(t) \propto 1-t\) & \bfseries  3.3 & \bfseries  22.1 \\
            & \(\dot{\theta}(t) \propto \exp((1-t)^{-1})\) & 4.3 & 23.7 \\
            & \(\dot{\theta}(t) \propto t(1-t)\) & 3.9 & 23.6 \\
            & \(\dot{\theta}(t) \propto \exp((t(1-t))^{-1})\) & 48.4 & 39.7 \\
            \bottomrule
        \end{tabular}
    }
    \vspace{-1em}
    \label{tbl:cot}
\end{table}


\subsection{Main results: 2D PDE-based Inference}
\paragraph{Datasets, architecture, and baselines.}
We follow \citet{li2020fourier} and evaluate on the Darcy Flow and Navier-Stokes equations on the 2D unit square.
We employ the U-shaped neural operator (U-NO; \citealp{rahman2022u}) for \({\varphi}_\theta\) and \({\eta}_\phi\).
We benchmark our approach against the original finite-dimensional SIs \citep{albergo2023stochasticinterpolantsunifyingframework} and two state-of-the-art diffusion models for solving PDEs: (i) FunDPS \citep{yao2025guideddiffusionsamplingfunction}, which considers functions in infinite-dimensional spaces, and; (ii) DiffusionPDE \citep{huang2024diffusionpde}, which uses a  finite-dimensional approach. We use a 49000/1000/1000 train/dev/evaluation split of paired functions discretised on \(64\!\times\!64\) evenly-spaced points on the unit square. To ensure a direct comparison, we use pre-processed data provided by \cite{huang2024diffusionpde}. Experimental details are provided in \cref{sec:exp-details}.
\paragraph{Experimental results.} 
\Cref{tbl:results} presents our results. The conditional bridge shows competitive results using both the ODE and SDE solvers. For the Navier-Stokes problem, both our infinite-dimensional methods outperform FunDPS and DiffusionPDE, with the ODE achieving the state-of-the-art.
 We attribute our underperformance on Darcy flow  to the representation of the permeability inputs \(a\) as a binary field, which introduces sharp discontinuities and high-frequency components that are poorly captured by  spectral convolutions  in the U-NO due to Gibbs phenomena \citep{de2022cost}. The U-Nets \citep{ronneberger2015u} used in finite dimensions are better suited to representing these discontinuities, likely contributing to their superior performance.

\begin{table}[t]
\centering
\sisetup{
  detect-weight,
  mode=text,
  parse-numbers = false
}
\renewcommand{\arraystretch}{1.15}
\setlength{\tabcolsep}{3.5pt}
\caption{Performance on PDE problems, measured in percentage (\%) relative \(L^{2}\)-error. Best results are marked in \textbf{bold}.}
\resizebox{\linewidth}{!}{
\begin{tabular}{
    lccccc
}
    \toprule
    \textbf{Methods}
    & \textbf{NFE}
    & \multicolumn{2}{c}{\textbf{Darcy Flow (\%)}}
    & \multicolumn{2}{c}{\textbf{Navier--Stokes (\%)}} \\
    \cmidrule(lr){3-4} \cmidrule(lr){5-6}
    & & \text{Forward} & \text{Inverse} & \text{Forward} & \text{Inverse} \\
    \midrule
    FunDPS \cite{yao2025guideddiffusionsamplingfunction} & 100 & 1.9 & 3.6 & 2.8 & 6.1 \\
    DiffusionPDE \cite{huang2024diffusionpde} & 2000 & 2.2 & 2.0 & 6.1 & 8.6 \\
    Finite-dim.\ SI + ODE & 100
        & \multicolumn{1}{c}{\bfseries 1.1}
        & \multicolumn{1}{c}{\bfseries 1.1}
        & 2.4 & 8.3 \\
    Finite-dim.\ SI + SDE & 100 & 1.5 & 4.9 & 1.7 & 7.1 \\
    \midrule
    Infinite-dim.\ SI + ODE (ours) & 100
        & 1.9 & 2.7
        & \multicolumn{1}{c}{\bfseries 1.0}
        & \multicolumn{1}{c}{\bfseries 4.6} \\
    Infinite-dim.\ SI + SDE (ours) & 100 & 2.3 & 3.7 & 1.4 & 6.1 \\
    \bottomrule
\end{tabular}
}
\label{tbl:results}
\end{table}

\section{Conclusion}
We proposed the first rigorous framework for stochastic interpolants in infinite-dimensional Hilbert spaces, resolving fundamental measure-theoretic and well-posedness challenges and demonstrating strong empirical performance on functional generative tasks.
While
our theoretical results rely on strong regularity conditions on the data distribution, primarily used to obtain rigorous guarantees and unlikely to hold completely in practice, relaxing these assumptions while retaining well-posedness and error control remains an important open problem. 
Furthermore
, a natural extension of our work is to consider datasets of functions with arbitrary and possibly sparse sensor locations.




\newpage
\section*{Acknowledgement}
JBY acknowledges the Cambridge Commonwealth, European and International Trust (CT) and Pembroke College for the Cambridge Trust \&  Pembroke College Studentship, which provided substantial support towards tuition fees during his MPhil in Machine Learning and Machine Intelligence, during which the research underlying this paper was conducted. RKOY acknowledges the UK Engineering and Physical Sciences Research Council (EPSRC) grant EP/L016516/1 for the University of Cambridge Centre for Doctoral Training, the Cambridge Centre for Analysis.
JMHL acknowledges support from EPSRC funding under grant EP/Y028805/1. JMHL also acknowledges support from a Turing AI Fellowship under grant EP/V023756/1.
This project acknowledges the resources provided by the Cambridge Service for Data-Driven Discovery (CSD3) operated by the University of Cambridge Research Computing Service (\href{www.csd3.cam.ac.uk}{www.csd3.cam.ac.uk}), provided by Dell EMC and Intel using Tier-2 funding from the Engineering and Physical Sciences Research Council (capital grant EP/T022159/1), and DiRAC funding from the Science and Technology Facilities Council (\href{www.dirac.ac.uk}{www.dirac.ac.uk}).

\section*{Impact Statement}
This paper presents work whose goal is to advance the field of Machine
Learning. There are many potential societal consequences of our work, none
which we feel must be specifically highlighted here.
\bibliography{paper}
\bibliographystyle{icml2026}

\newpage
\appendix
\onecolumn
\crefalias{section}{appendix}
\crefalias{subsection}{appendix}

\vbox{%
    \hsize\textwidth
    \linewidth\hsize
    \vskip 0.1in
    \hrule height 4pt
  \vskip 0.25in
  \vskip -\parskip%
    \centering
    {\LARGE\bf {Stochastic Interpolants in Hilbert Spaces\\ Appendix} \par}
     \vskip 0.29in
  \vskip -\parskip
  \hrule height 1pt
  \vskip 0.09in%
  }

\section*{Outline of Appendix}
\begin{itemize}
    \item Appendix \ref{app:weak-and-strong} introduces weak and strong solutions to SDEs;
    \item Appendix \ref{sec:backwards} introduces the backward SDE for the Conditional-Bridge SDE;
    \item Appendix \ref{app:fixed-epsilon} discusses the time-independent diffusion term, $\epsilon$, in our proposed framework;
    \item Appendix \ref{sec:proofs} provides proofs for lemmas, propositions, and theorems;
    \item Appendix \ref{sec:dp} provides the instantiation of the propsoed framework;
    \item Appendix \ref{sec:algos} provides pseudo-codes for both training and sampling procedures;
    \item Appendix \ref{sec:exp-details} provides experimental details, including datasets and hyperparameters;
    \item Appendix \ref{app:figures} provides additional figures on 2D PDE problems.
\end{itemize}

\section{Weak and strong solutions to SDEs} \label{app:weak-and-strong}
In this section, we define weak and strong solutions to the SDE in \cref{eqn:hsde}. We refer the reader to \citet{da2014stochastic,scheutzow2013stochastic} for an in-depth treatment. A \textit{weak solution} to \Cref{eqn:hsde} is a triple formed of
\begin{enumerate}
  \item filtered probability space \((\Omega, \mathcal{F}, \mathbb{F}, \mu)\) with normal filtration on \([0, 1]\);
  \item a continuous \(\mathbb{F}\)-adapted stochastic process \(X_{t}\) with \(X_{0} \sim \mu_{0}\); and
  \item an \(\mathbb{F}\)-adapted \(C\)-Wiener process \(W_{t}\),
\end{enumerate}
such that \(\mu\)-almost surely,
\begin{equation}
  X_{t} = X_{0} + \int_{0}^{t} f(s, X_{s}) \dd{s} + \int_{0}^{t} g(s, X_{s}) \dd{W_{s}}
\end{equation}
for all \(t \in [0, 1]\). We say an SDE has \textit{unique weak solutions} if all weak solutions have the same law for all \(t \in [0,1 ]\).

For a \textit{given} filtered probability space \((\Omega, \mathcal{F}, \mathbb{F}, \mu)\), a \textit{given} \(C\)-Wiener process and a \textit{given} initial condition \(x_{0} \sim \mu_{0} \perp W_{t}\), we say \(X_{t}\) is a \textit{strong solution} to \Cref{eqn:hsde} if:
\begin{enumerate}
  \item \(((\Omega, \mathcal{F}, \mathbb{F}, \mu), X_{t}, W_{t})\) are a weak solution;
  \item \(X_{0} = x_{0}\), \(\mu\)-almost surely; and
  \item \(X_{t}\) is adapted to the smallest normal filtration generated by \((x_{0}, W_{t})\).
\end{enumerate}

An SDE has a \textit{unique strong solution} if all strong solutions are equal \(\mu\)-almost surely. Uniqueness of strong solutions allows us to couple two stochastic process driven by a common pre-specified \(C\)-Wiener process, which we use to prove error bounds in \Cref{thm:w2}.

\section{Bridging from Target to Source} \label{sec:backwards}
Our preceding analysis has focused on establishing a stochastic bridge from the source \(x_{0} \sim \mu_{0}\) to the conditional target distribution \(\mu_{1 \mid 0}(\dd{x_{1}}, x_{0})\). To bridge from a target point \(x_{1} \sim \mu_{1}\) to the conditional source distribution \(\mu_{0 \mid 1}(\dd{x_{0}}, x_{1})\), we may consider a \textit{reverse interpolant}
\begin{equation}
x_{t}^{\text{rev}} \coloneqq \alpha(1-t) x_{0} + \beta(1-t) x_{1} + \gamma(1-t).
\end{equation}
The analogous SDE bridging from \(x_{1}\) to \(\mu_{0 \mid 1}(\dd{x_{0}}, x_{1})\), which we call the \textit{reverse conditional bridge SDE} (RCB-SDE), is
\begin{equation}
\dd{X_{t}}^{\text{rev}} = f^{\text{rev}}(t, x_{1}, X_{t}^{\text{rev}}) \dd{t} + \sqrt{2\varepsilon}\dd{W_{t}}, \quad X_{0}^{\text{rev}} = x_{1},
\end{equation}
where
\begin{align}
f^{\text{rev}}(t, x_{1}, x) &\coloneqq \varphi^{\text{rev}}(t, x_{0}, x) - \qty( \dot{\gamma}(1-t) + \frac{\varepsilon}{\gamma(1-t)} ) \eta^{\text{rev}}(t, x_{0}, x) \\
\varphi^{\text{rev}}(t, x_{1}, x) &\coloneqq -\mathop{\mathbb{E}}\qty[\dot{\alpha}(1-t) x_{0} + \dot{\beta}(1-t) x_{1} \mid x_{1} = x_{1}, x^{\text{rev}}_{t} = x], \\
\eta^{\text{rev}}(t, x_{1}, x) &\coloneqq \mathop{\mathbb{E}}\qty[z \mid x_{1} = x_{1}, x_{t} = x].
\end{align}
Unlike in \citet{albergo2023stochasticinterpolantsunifyingframework}, we state this as an SDE solved \textit{forward} in time starting from initial condition \(X_{0}^{\text{rev}} = x_{1}\). This gives a consistent indexing of time, and assuming \(\gamma(t) = \gamma(1-t)\), the same time change \(\theta(t)\) can be applied. All preceding results apply analogously by re-stating the conditions on \(x_{1}\) as conditions on \(x_{0}\). Similarly, the \textit{reverse marginal bridge SDE} (RMB-SDE) can be recovered by removing conditioning on \(x_{1}\) in the definitions above. 

Since conditional expectations are taken conditional on \((x_{1}, x^{\text{rev}}_{t})\), we must in general train additional networks to learn both the CB-SDE and the RCB-SDE. In contrast, if only the \textit{marginal} bridges matter, the same trained networks can be used for both forward and reverse tasks by the deterministic relationship \(x_{t} = x_{1-t}^{\text{rev}}\). Our results (\Cref{tbl:cot}) show that simulating a marginal bridge for conditional tasks causes only a modest loss in performance, and hence the marginal bridge may be still be useful in conditional settings.



\section{On Time-Constant \(\varepsilon\) and Time Change} \label{app:fixed-epsilon}
    In our framework, we restrict \(\varepsilon\) to be constant, which is a regularity condition present in most well-posedness results for infinite-dimensional Fokker-Planck equations, which assume a time-independent coefficient on \(W_{t}\) \citep{bogachev2009fokker,bogachev2010uniquenesssolutionsfokkerplanckequations}. In finite-dimensional SIs, \citet{albergo2023stochasticinterpolantsunifyingframework} allow \(\varepsilon(t)\) to vary in time, but we do not view the constant-\(\varepsilon\) setting as restrictive. It can be shown that for a fixed \(\varepsilon\), the time-changed process \(\hat{X}_{\theta(t)}\) is equal in law to a \textit{new} SI with different schedules \(\alpha(t), \beta(t), \gamma(t)\) and a time-varying \(\varepsilon(t)\). A change-of-time hence causes the magnitude of injected noise to vary over time, but only through a simultaneous time reparameterization of the interpolation between source and target, rather than allowing an independent choice of the noise schedule. While this does not fully decouple the diffusion from drift schedules, we believe that this restriction is helpful for defining stochastic interpolants with stable sampling dynamics. In particular, our framework simplifies the design space of stochastic interpolants in a way that helps to decouple \textit{interpolating trajectories} (by choosing the scale of \(\varepsilon\)) from numerical \textit{sampling stability} (by choosing the change-of-time \(\theta\)).
\section{Proofs}\label{sec:proofs}
\subsection{Proof of \Cref{thm:cbsde}}
\label{prf:thm:cbsde}
\restatethmcbsde*

\begin{proof}
Suppose the CB-SDE has a solution on a (possibly strict) subinterval \([0, \overline{t}] \subseteq [0, 1]\) where the law of \(X_{t}\) for \(t \in [0, \overline{t}]\), conditional on an initial condition \(x_{0} \sim \mu_{0}\) is unique and given by \(\rho_{t \mid 0}(\cdot, x_{0})\). Standard results \citep[Chapter 14.2.2]{da2014stochastic} show that for \(\mu_{0}\)-almost every \(x_{0} \in H\) and \(\dd{t}\)-almost every \(t \in [0, \overline{t}]\),  \(\rho_{t \mid 0}(\cdot \mid x_{0})\) satisfies the following \textit{Fokker-Planck} equation:
\begin{equation}
  \dv{t} \int_{H} u(t, x) \rho_{t \mid 0}(\dd{x} \mid x_{0}) = \int_{H} \mathcal{L}(x_{0}) u(t, x) \rho_{t \mid 0}(\dd{x} \mid x_{0}), \label{eqn:fp}
\end{equation}
for all test functions \(u(t, x)\) in the space \(E\) formed by the linear span of the real and imaginary components of functions of the following form:
\begin{equation}
  E \coloneqq \qty{u : [0, \overline{t}] \times H \to \mathbb{R} \mid u(t, x) = \phi(t) e^{i \ev{x, h(t)}_{H}},  \text{ for some }\phi \in C^{1}([0, \overline{t}]), h \in C^{1}([0, \overline{t}]; H)}, \label{eqn:testfns}
\end{equation}
and where  \(\mathcal{L}(x_{0})\) is a \textit{Kolmogorov operator} given by:
\[
  \mathcal{L}(x_{0}) u(t, x) \coloneqq  \operatorname{Tr}\qty(\varepsilon C D^{2}_{x} u(t, x)) + D_{t}u(t, x) + \ev{f(t, x_{0}, x), D_{x} u(t, x)}_{H}.
\]

We use \(D_{t}\) to denote the derivative in time, and \(D_{x}, D^{2}_{x}\) the first and second-order Fréchet derivatives in Hilbert space. 



This formulation via test functions is standard for SDEs in infinite dimensions where a uniform reference measure may not exist. Since this Fokker-Plank equation is uniquely solvable, it is sufficient to verify that the interpolant's conditional distribution, \(\mu_{t \mid 0}(\cdot \mid x_{0})\), also satisfies this Fokker-Planck equation for all test functions \(u \in E\). Throughout this proof, we fix \(x_{0}\) arbitrarily drawn from \(\mu_{0}\).

  It is sufficient to consider any test function of the form \(u(t, x) = \operatorname{Re}\qty[\phi(t) e^{i \ev{x, h(t)}_{H}}]\) or \(\operatorname{Im}\qty[\phi(t) e^{i \ev{x, h(t)}_{H}}]\), where \(\phi\) and \(h\) satisfy the properties given in \Cref{eqn:testfns}. 

  Fix \(t \in [0,1]\) and consider the characteristic function of the random variable \(u(t, x_{t})\), conditional on \(x_{0}\). For any \(k \in \mathbb{R}\), we define
  \begin{align}
    \chi(t, k \mid x_{0}) &\coloneqq \mathop{\mathbb{E}}\qty[e^{i k u(t, x_{t})} \mid x_{0}]
  \end{align}
  Taking derivatives with respect to \(t\) and \(k\) and evaluating at \(k=0\) recovers the time derivative of the expectation of \(u(t, x_{t})\):
  \begin{equation}
    \eval{\frac{1}{i} \pdv[2]{}{t}{k} \chi(t, k \mid x_{0} )}_{k=0} = \dv{t} \mathop{\mathbb{E}}\qty[u(t, x_{t})] = \mathop{\mathbb{E}}\qty[D_{t} u(t, x_{t}) + \ev{ \dot{x}_{t}, D_{x}u(t, x_{t})}_{H}].
  \end{equation}

  Since the inner product \(\ev{ \dot{x}_{t}, D_{x} u(t, x_{t})}_{H}\) is linear in its first argument, we may apply the law of iterated expectations and replace \(\dot{x}_{t}\) with \( \mathop{\mathbb{E}}\qty[ \dot{x}_{t} \mid x_{t}]\):
  \[
    \dv{t} \mathop{\mathbb{E}}\qty[u(t, x_{t}) \mid x_{0}] = \mathop{\mathbb{E}}\qty[D_{t} u(t, x_{t}) + \ev{  \mathop{\mathbb{E}}\qty[\dot{x}_{t} \mid x_{t}], D_{x} u(t, x_{t})}_{H} \mid x_{0} ]
  \]

  We add and subtract \(\frac{\varepsilon}{\gamma(t)} \eta(t, x_{0}, x_{t})\), where \(\eta(t, x_{0}, x_{t}) = \mathop{\mathbb{E}}\qty[ z \mid x_{0}, x_{t}]\) as defined in \Cref{eq:conditional-bridge-denoiser}, and spell out \(\dot{x}_{t} = \dot{\alpha}(t)x_{0} + \dot{\beta}(t)x_{1}+ \dot{\gamma}(t) z\):
  \begin{align}
    \dv{t} \mathop{\mathbb{E}}\qty[u(t, x_{t}) \mid x_{0} ] &= \mathop{\mathbb{E}}\Bigg[D_{t} u(t, x_{t}) + \Bigg\langle \frac{\varepsilon}{\gamma(t)}\eta(t, x_{t}) + \mathop{\mathbb{E}}\qty[\dot{\alpha}(t) x_{0} + \dot{\beta}(t) x_{1} + \dot{\gamma}(t) z \mid x_{0}] \notag  \\
    &\mathrel{\phantom{=}}\qquad\qquad\qquad\qquad\qquad\qquad\qquad\qquad - \,\frac{\varepsilon}{\gamma(t)} \eta(t, x_{t}), D_{x} u(t, x_{t})\Bigg\rangle_{H} \,\Bigg\vert\, x_{0} \Bigg] \notag \\
    &= \frac{\varepsilon}{\gamma(t)}\mathop{\mathbb{E}}\qty[\ev{z, D_{x}u(t, x_{t})}_{H} \mid x_{0}] + \mathop{\mathbb{E}}\qty[ D_{t}u(t, x_{t}) + \ev{f(t, x_{t}), D_{x} u(t, x_{t})}_{H} \mid x_{0}]. \label{eqn:simpl}
  \end{align}

  For the following, we assume that \(u(t, x) = \operatorname{Re}[\phi(t) e^{i \ev{x, h(t)}_{H}}]\), but identical reasoning applies if \(u(t, x) = \operatorname{Im}\qty[\phi(t) e^{i \ev{x, h(t)}_{H}}]\).

  Let us focus on the first term in \Cref{eqn:simpl}. We have:
  \par\noindent
  \resizebox{\linewidth}{!}{
    \begin{minipage}{\linewidth}
      \begin{align}
        \frac{\varepsilon}{\gamma(t)} \mathop{\mathbb{E}}[\ev{z, D_{x} u(t, x_{t})}_{H} \mid x_{0}] &= \operatorname{Re}\qty[i \frac{\varepsilon}{\gamma(t)} \mathop{\mathbb{E}}\qty[\phi(t) e^{i \ev{x_{t}, h(t)}_{H}} \ev{z, h(t)}_{H} \mid x_{0}]] \notag \\
        &= \operatorname{Re}\qty[i \frac{\varepsilon}{\gamma^{2}(t)} \mathop{\mathbb{E}}\qty[\phi(t) e^{i \ev{\alpha(t) x_{0} + \beta(t) x_{1}, h(t)}_{H}}  \mid x_{0}] \mathop{\mathbb{E}}\qty[\ev{\gamma(t) z, h(t)}_{H}e^{i \ev{\gamma(t) z, h(t)}_{H}} ]], \label{eqn:nearlythere}
      \end{align}
  \end{minipage}}
  \par\noindent
  where the second line follows since \(z \perp (x_{0}, x_{1})\).

  Let \(\qty{\lambda_{n}, e_{n}}_{n=1}^{\infty}\) be an eigensystem for \(C\) (i.e. \(C e_{n} = \lambda e_{n}\) for each \(n\)) and define the scalar-valued functions \(h_{n}(t) \coloneqq \ev{h(t), e_{n}}_{H}\). The projections \(z_{n} = \ev{z, e_{n}}\) for each \(n\) are mutually independent 1-dimensional Gaussians with zero mean and variances equal to \(\lambda_{n}\).  By Parseval's theorem, we have the identity \(\ev{\gamma(t) z, h(t)} = \sum_{n=1}^{\infty} \gamma(t) h_{n}(t) z_{n} \). We may therefore write
  \[
    \mathop{\mathbb{E}}\qty[\ev{\gamma(t) z, h(t)}_{H} e^{i \ev{\gamma(t)z, h(t)}_{H}}] = \sum_{n=1}^{\infty} \mathop{\mathbb{E}}\qty[\gamma(t) h_{n}(t) z_{n} e^{i \gamma(t) h_{n}(t) z_{n}}] \prod\limits_{m \neq n} \mathop{\mathbb{E}}\qty[e^{i \gamma(t) h_{m}(t) z_{m}}]
  \]
  Using the identity \(\mathop{\mathbb{E}}\qty[q e^{i q}] = i v \mathop{\mathbb{E}}\qty[e^{i q}]\) for a 1-dimensional Gaussian \(q \sim \operatorname{N}(0, v)\), we have
  \begin{align*}
    \mathop{\mathbb{E}}\qty[\ev{\gamma(t) z, h(t)}_{H} e^{i \ev{\gamma(t)z, h(t)}_{H}}] = \sum_{n=1}^{\infty} i \gamma^{2}(t)h_{n}^{2}(t)\lambda_{n} \mathop{\mathbb{E}}\qty[e^{i \ev{\gamma(t) z, h(t)}_{H}}]
  \end{align*}
  Substituting into \Cref{eqn:nearlythere}, we have
  \begin{align*}
    \frac{\varepsilon}{\gamma(t)} \mathop{\mathbb{E}}\qty[\ev{z, D_{x} u(t, x_{t})}_{H} \mid x_{0}] &= \mathop{\mathbb{E}}\qty[\sum_{n=1}^{\infty} - \varepsilon \lambda_{n} h_{n}^{2}(t) u(t, x_{t}) \,\Bigg\vert\, x_{0}] = \mathop{\mathbb{E}}\qty[\operatorname{Tr}\qty(\varepsilon C D^{2}_{x}u(t, x_{t})) \mid x_{0}].
  \end{align*}
  Finally, substituting this expression into \Cref{eqn:simpl} and re-writing conditional expectations via integrals, we have
  \[
    \dv{t} \int_{H} u(t, x) \mu_{t \mid 0}(\dd{x} \mid x_{0}) = \int_{H} \operatorname{Tr}\qty(\varepsilon C D^{2}_{x} u(t, x)) + D_{t}u(t, x) + \ev{f(t, x), D_{x}u(t, x)}_{H} \mu_{t \mid 0}(\dd{x} \mid x_{0}).
  \]
  Since  \(t\) and \(x_{0}\) were arbitrary,  \(\mu_{t \mid 0}(\cdot, \mid x_{0}\) satisfies the Fokker-Planck equation (\ref{eqn:fp}) for any \(t \in [0, 1]\) and \(\mu_{0}\) almost every \(x_{0}\). This concludes the proof.

\end{proof}

\subsection{Proof of \Cref{lem:lipchitz-conditions}\ref{setting:bayes}}\label{prf:prp:bayes}
\paragraph{Proposition \ref{lem:lipchitz-conditions}\ref{setting:bayes}.}
{
\itshape
    Suppose \Cref{hyp:bayes} holds. Then the map \(x \mapsto f(t, x_{0}, x)\) is Lipschitz continuous with respect to the \(H_{C}\)-norm. Specifically, for each \(t \in (0, 1)\), \(x_{0} \in H_{C}\) and \(x \in H\), the following inequality holds:
    \[
      \norm{f(t, x_{0}, x) - f(t, x_{0}, y)}_{H_{C}} \leq L(t) \norm{x - y}_{H_{C}},
    \]
    where the Lipschitz constant \(L(t)\) is:
    \[
      L(t) =  \mathop{\underset{}{\max}}\qty{\abs{\frac{\dot{\gamma}(t)}{\gamma(t)} - \frac{\varepsilon}{\gamma^{2}(t)}}, \abs{\dot{\beta}(t) - \beta(t)\qty(\frac{\dot{\gamma}(t)}{\gamma(t)} - \frac{\varepsilon}{\gamma^{2}(t)})} \frac{\beta(t)}{\beta^{2}(t) + k \gamma^{2}(t)}}.
    \]
}

\begin{proof}
  First, we notice that the drift term can be re-written:
  \begin{align}
    f(t, x_{0}, x) &= \mathop{\mathbb{E}}\qty[\dot{\alpha}(t) x_{0} + \dot{\beta}(t) x_{1} + \qty(\dot{\gamma}(t) - \frac{\varepsilon}{\gamma(t)}) z \,\bigg|\, x_{0}, x_{t} = x] \notag\\
    &= \qty(\dot{\alpha}(t) - \alpha(t)\qty(\frac{\dot{\gamma}(t)}{\gamma(t)} - \frac{\varepsilon}{\gamma^{2}(t)})) x_{0} \notag\\
    &\mathrel{\phantom{=}}\, + \qty(\dot{\beta}(t) - \beta(t)\qty(\frac{\dot{\gamma}(t)}{\gamma(t)} - \frac{\varepsilon}{\gamma^{2}(t)})) \mathop{\mathbb{E}}\qty[x_{1} \mid x_{0}, x_{t} = x]\notag \\
    &\mathrel{\phantom{=}}\, + \qty(\frac{\dot{\gamma}(t)}{\gamma(t)} - \frac{\varepsilon}{\gamma^{2}(t)}) x_{t}.\label{eqn:driftreexpressed}
  \end{align}
  Hence, if we can show that the mapping \(x \mapsto \mathop{\mathbb{E}}\qty[x_{1} \mid x_{0}, x_{t} = x]\) is Lipschitz continuous in \(H_{C}\)-norm, this translates to Lipschitz continuity in the overall mapping \(x \mapsto f(t, x_{0}, x)\).

  \paragraph{Step 1}
  Let \(\operatorname{N}(m_{1 \mid 0}(x_{0}), Q_{1 \mid 0})\) be the (prior) conditional law of \(x_{1}\), conditional on \(x_{0}\), which is a well-defined Gaussian measure on \(H_{C}\)  (see, e.g., \citealp[][Chapter 3.10]{bogachev1998gaussian}) whose mean \(m_{1 \mid 0}(x_{0})\) is a linear function of \(x_{0}\) and covariance operator \(Q_{1 \mid 0}\) does not depend on \(x_{0}\). This will serve as our reference measure. Let \(\mu_{1 \mid 0, t}(\cdot \mid x_{0}, x)\) be the (posterior) conditional law of \(x_{1}\), conditional on \(x_{0}\) and \(x_{t} = x\).  Then for \(\mu_{0}\)-almost every \(x_{0} \in H_{C}\), the law \(\mu_{1 \mid 0, t}(\cdot \mid x_{0}, x)\) is absolutely continuous with respect to the reference measure with the following density
  \begin{align*}
    \dv{\mu_{1 \mid 0, t}(\cdot \mid x_{0}, x)}{\operatorname{N}(m_{1 \mid 0}(x_{0}), Q_{1 \mid 0})}{}(x_{1}) &= \frac{1}{Z_{1 \mid 0,t}(x_{0}, x)}\exp(- V_{1 \mid 0, t}(x_{1}, x_{0}, x)),\\
    \text{ where } V_{1 \mid 0, t}(x_{1}, x_{0}, x) &\coloneqq \frac{1}{2\gamma^{2}(t)} \norm{\alpha(t) x_{0}+ \beta(t) x_{1} - x}_{H_{C}}^{2} + \Phi(x_{0}, x_{1}),
  \end{align*}
  and \(Z_{1 \mid 0,t}(x_{0}, x) \coloneqq \int_{H_{C}} \exp(- V_{1 \mid 0, t}(x_{1}, x_{0}, x)) \operatorname{N}(\dd{x_{1}}; m_{1 \mid 0}(x_{0}), Q_{1 \mid 0})\) is a normalising constant.
  \paragraph{Step 2} Let \(\qty{e_{n}}_{n=1}^{\infty}\) be an orthonormal basis for \(H_{C}\) and for each \(N \geq 1\), let \(H_{N}\) be the linear span of \(\qty{e_{1}, \ldots, e_{N}}\). We define \(\Pi_{N} : H_{C} \to H_{N}\) as the self-adjoint orthogonal projection operator onto the finite-dimensional subspace \(H_{N}\) of \(H_{C}\) and let \(x_{1,N} \coloneqq \Pi_{N} x_{1}\). Furthermore, we define a reference measure on \(H_N\) by projecting \(\operatorname{N}(m_{1 \mid 0}(x_{0}), Q_{1 \mid 0})\) onto this subspace to create a Gaussian measure with mean \(m_{1 \mid 0, N}(x_{0}) \coloneqq \Pi_{N} m_{1 \mid 0}(x_{0})\) and variance \(Q_{N} \coloneqq \Pi_{N} Q_{1 \mid 0} \Pi _{N}\).
  
  Using this, we create a sequence of approximating posterior measures \(\mu_{1 \mid 0, t, N}\) by restricting the potential to \(H_{N}\): for each \(x_{1,N} \in H_{N}\):
  \begin{align*}
    \dv{\mu_{1 \mid 0, t, N}(\cdot\mid x_{0}, x)}{\operatorname{N}(m_{1 \mid 0, N}(x_{0}), Q_{N})}{}(x_{1,N}) &\coloneqq \frac{1}{Z_{1 \mid 0, t, N}(x_{0}, x)} \exp(-V_{1 \mid 0, t, N}(x_{1,N}, x_{0}, x)), \\
    \text{ where } V_{1 \mid 0, t, N}(x_{1, N}, x_{0}, x) &\coloneqq  \frac{1}{2\gamma^{2}(t)} \norm{\alpha(t) \Pi_{N} x_{0} + \beta(t) x_{1, N} - x}^{2}_{H_{C}} + \Phi(x_{0}, x_{1, N}),
  \end{align*}
  where \(Z_{1 \mid 0, t, N}(x_{0}, x) \coloneqq \int_{H_{N}} \exp(-V_{1 \mid 0, t, N})(x_{1, N}, x_{0}, x)\operatorname{N}(\dd{x_{1}}; m_{1 \mid 0, N}(x_{0}), Q_{N}) \) is a normalising constant.

  Given these definitions, we study the following approximation of the posterior mean:
  \begin{equation}
    m_{1 \mid 0, t, N}(x_{0}, x) \coloneqq \mathop{\mathbb{E}_{\mu_{1 \mid 0, t, N}(\cdot, x_{0}, x)}}\qty[x_{1, N}] = \int_{H_{N}} x_{1, N}\mu_{1 \mid 0, t, N}(\dd{x_{1, N}} \mid x_{0}, x).\label{eqn:apm}
  \end{equation}
  We find a Lipschitz constant for the map \(x \mapsto m_{1 \mid 0, t, N}(x_{0}, x)\) that is independent of \(N\) and \(x_{0}\), by considering the Fréchet derivative of \(m_{1 \mid 0, t, N}(x_{0}, x)\) with respect to \(x\), applied in a direction \(h \in H_{C}\). This is a covariance (see Lemma \ref{lem:frechetf}):
  \par\noindent
  \resizebox{\linewidth}{!}{
    \begin{minipage}{\linewidth}
      \begin{align}
        D_{x} m_{1 \mid 0, t, N}(x_{0}, x)[h] &=\frac{\beta(t)}{\gamma^{2}(t)} \mathop{\mathbb{E}_{\mu_{1 \mid 0, t, N}(\cdot, x_{0}, x)}}\qty[(x_{1, N} - m_{1 \mid 0, t, N}(x_{0}, x)) \ev{x_{1, N} - m_{1 \mid 0, t, N}(x_{0}, x), h}_{H_{C}}] \notag \\
        &=\frac{\beta(t)}{\gamma^{2}(t)} \mathop{\mathbb{E}_{\mu_{1 \mid 0, t, N}(\cdot, x_{0}, x)}}\qty[(x_{1, N} - m_{1 \mid 0, t, N}(x_{0}, x)) \ev{x_{1, N} - m_{1 \mid 0, t, N}(x_{0}, x), \Pi_{N} h}_{H_{N}}] \label{eqn:dxmt0},
      \end{align}
  \end{minipage}}
  \par\noindent
  where the second equality follows from the first since the components of \(x_{1, N}\allowbreak{}-\allowbreak{}m_{1 \mid 0, \allowbreak{}t, \allowbreak{}N}(x_{0}, \allowbreak{}x)\) along the basis vectors \(\qty{e_{n}}_{n=N+1}^{\infty}\) are all zero.

  By the Riesz representation theorem \citep{bachman2000functional}, the \(N\)-dimensional subspace \(H_{N}\) is isomorphic with \(\mathbb{R}^{N}\), so all vectors on \(H_{N}\) can be identified with an \(N\)-dimensional column vector in \(\mathbb{R}^{N}\). We therefore re-write the derivative using an \(N\)-dimensional covariance matrix \(C_{N}\) acting on  \(\Pi_{N} h\):
  \begin{align}
    D_{x} m_{1 \mid 0, t, N}(x_{0}, x)[h] &= \frac{\beta(t)}{\gamma^{2}(t)} C_{N} \Pi_{N} h, \label{eqn:dxmt}\\
    \text{ where } C_{N} &= \mathop{\mathbb{E}_{\mu_{1 \mid 0, t, N}(\cdot, x_{0}, x)}}\qty((x_{1, N} - m_{1 \mid 0, t, N}(x_{0}, x))(x_{1, N} - m_{1 \mid 0, t, N}(x_{0}, x))^{\tran}). \notag
  \end{align}
  For the rest of the proof, we identify \(C_{N}\) with a self-adjoint covariance operator on \(H_{N}\).

  \paragraph{Step 3}
  We use the Brascamp-Lieb inequality \citep{brascamp1976extensions} to  bound the operator norm of \(C_{N}\). We proceed by expressing  \(\mu_{1 \mid 0, t, N}(\cdot \mid x_{0}, x)\) via a density relative to the Lebesgue measure on \(H_{N}\) (identified with \(\mathbb{R}^{N}\)). The density of the reference measure \(\operatorname{N}(m_{1 \mid 0, N}(x_{0}), Q_{N})\) with respect to the Lebesgue measure, evaluated at \(x_{1, N} \in H_{N}\), is proportional to
  \[\exp(-\frac{1}{2} \ev{Q_{N}^{-1} (x_{1, N} - m_{1 \mid 0, N}(x_{0})), x_{1, N} - m_{1 \mid 0, N}(x_{0})}_{H_{N}}),
  \]
  where the inverse \(Q_{N}^{-1}\) is well-defined because \(Q_{N} : H_{N} \to H_{N}\) is positive-definite and bounded. Hence, 
  \begin{align*}
    &\mu_{1 \mid 0, t, N}(\dd{x_{1, N}} \mid x_{0}, x) \\
    &\quad\propto \exp(-V_{1 \mid 0, t, N}(x_{1, N}, x_{0}, x) - \frac{1}{2} \ev{Q_{N}^{-1} (x_{1, N} - m_{1 \mid 0, N}(x_{0})), x_{1, N} - m_{1 \mid 0, N}(x_{0})}_{H_{N}}) \dd{x_{1, N}}.
  \end{align*}
  Let
  \[W_{1 \mid 0, t, N}(x_{1, N}, x_{0}, x) \coloneqq V_{1 \mid 0, t, N}(x_{1, N}, x_{0}, x) + \frac{1}{2} \ev{Q_{N}^{-1} (x_{1, N} - m_{1 \mid 0, N}(x_{0})), x_{1, N} - m_{1 \mid 0, N}(x_{0})}_{H_{N}}\]  be the total potential with respect to the Lebesgue measure on \(H_{N}\). Since this is twice-differentiable and strictly convex, the conditions for the Brascamp-Lieb inequality are satisfied \citep[Theorem 4.1]{brascamp1976extensions}: for any continuously differentiable function \(f : H_{N} \to \mathbb{R}\), we have
  \begin{align*}
    &\mathop{\mathbb{E}_{\mu_{1 \mid 0, t, N}(\cdot, x_{0}, x)}}\qty[\qty(f(x_{1, N}) - \bar{f})^{2}]\\
    &\quad\leq \mathop{\mathbb{E}_{\mu_{1 \mid 0, t, N}(\cdot, x_{0}, x)}}\qty[\ev{ \qty(D^{2}_{x_{1, N}}W_{1 \mid 0, t, N}(x_{1, N}, x_{0}, x))^{-1} D f(x_{1, N}), D f(x_{1, N}) }_{H_{N}}],
  \end{align*}
  where \(\overline{f}\) is the expectation of \(f(x_{1, N})\) under the measure \(\mu_{1 \mid 0, t, N}(\cdot \mid x_{0}, x)\) and \(D^{2}_{x_{1, N}}\) is  Hessian operator with respect to \(x_{1, N}\) on \(H_{N}\). In the case where \(f(x_{1, N}) = \ev{x_{1, N}, u}_{H_{N}}\) for any \(u \in H_{N}\), we have \(Df(x_{1, N}) = u\), and
  \begin{align}
    &\mathop{\mathbb{E}_{\mu_{1 \mid 0, t, N}(\cdot, x_{0}, x)}}\qty[\qty(f(x_{1, N}) - \bar{f})^{2}] = \ev{C_{N}u, u} \notag\\
    &\qquad\qquad\leq  \mathop{\mathbb{E}_{\mu_{1 \mid 0, t, N}(\cdot, x_{0}, x)}}\qty[\ev{ \qty(D^{2}_{x_{1, N}}W_{1 \mid 0, t, N}(x_{1, N}, x_{0}, x))^{-1} u, u }_{H_{N}}]. \label{eqn:brascamp}
  \end{align}
  \paragraph{Step 4}
  We aim to place a Löwner order on the inverse Hessian \(\qty(D^{2}_{x_{1, N}}W_{1 \mid 0, t, N}(x_{1, N}, x_{0}, x))^{-1}\) irrespective of \(x_{1, N}\), which will  allow us to form a Löwner order on \(C_{N}\).

  Taking the second-order Fréchet derivatives of \(W_{1 \mid 0, t,N}(x_{1, N}, x_{0}, x)\) with respect to \(x_{1, N}\) in the directions \(u, v \in H_{N}\), we have
  \begin{align*}
    D^{2}_{x_{1, N}}W_{1 \mid 0, t, N}(x_{N}, x_{0}, x)[u, v] = \ev{\qty(\frac{\beta^{2}(t)}{\gamma^{2}(t)} I_{N} + \Pi_{N}\grad_{x_{1}}^{2}\Phi(x_{0}, x_{1, N})\Pi_{N} + Q_{N}^{-1})u, v}_{H_{N}},
  \end{align*}
  where \(\grad_{x_{1}}^{2} \Phi(x_{0}, x_{1})\) is the partial Hessian of the potential \(\Phi\) with respect to the second coordinate. This allows us to identify the Hessian with a self-adjoint Hessian operator from \(H_{N}\) to \(H_{N}\):
  \begin{equation}
    D^{2}_{x_{1, N}}W_{1 \mid 0, t, N}(x_{N}, x_{0}, x)[u, v] = \frac{\beta^{2}(t)}{\gamma^{2}(t)} I_{N} + \Pi_{N}\grad_{x_{1}}^{2}\Phi(x_{0}, x_{1, N})\Pi_{N} + Q_{N}^{-1}\label{eqn:hess}
  \end{equation}
  Since \(\Phi\) is \(k\)-strongly convex, it is also \(k\)-strongly convex in the second coordinate and hence the projection of its partial Hessian satisfies the following Löwner order:
  \[
    \Pi_{N} \grad^{2}_{x_{1}} \Phi(x_{0}, x_{1, N}) \succcurlyeq k I_{N},
  \]
  which allows us to place a Löwner order on \Cref{eqn:hess}:
  \[
    D^{2}_{x_{1, N}}W_{1 \mid 0, t, N}(x_{N}, x_{0}, x)[u, v] \succcurlyeq \qty(\frac{\beta^{2}(t)}{\gamma^{2}(t)} + k) I_{N} + Q_{N}^{-1}
  \]
  Since the right-hand side of this quantity is positive-definite, this Löwner order is reversed when taking inverses:
  \[
    \qty(D^{2}_{x_{1, N}}W_{1 \mid 0, t, N}(x_{N}, x_{0}, x)[u, v])^{-1} \preccurlyeq \qty(\qty(\frac{\beta^{2}(t)}{\gamma^{2}(t)} + k) I_{N} + Q_{N}^{-1})^{-1}.
  \]
  This relationship holds uniformly for all \(x_{1, N} \in H_{N}\). Substituting into \Cref{eqn:brascamp}, we have
  \begin{align*}
    \ev{C_{N}u, u} &\leq \ev{\qty(\qty(\frac{\beta^{2}(t)}{\gamma^{2}(t)} + k) I_{N} + Q_{N}^{-1})^{-1}u, u}_{H_{N}}, \text{ for all } u \in H_{N}\\
    \iff C_{N} &\preccurlyeq \qty(\qty(\frac{\beta^{2}(t)}{\gamma^{2}(t)} + k) I_{N} + Q_{N}^{-1})^{-1}.
  \end{align*}
  \paragraph{Step 5} Having established a Löwner order on \(C_{N}\), we now use this to place a bound on the operator norm of \(C_{N}\). Since \(C_{N}\) is positive semi-definite, the Löwner order translates directly into an ordering on operator norms:
  \[
    \onorm{C_{N}} \leq \onorm{\qty(\qty(\frac{\beta^{2}(t)}{\gamma^{2}(t)} + k) I_{N} + Q_{N}^{-1})^{-1}}.
  \] 
  The spectrum of \(\qty(\qty(\frac{\beta^{2}(t)}{\gamma^{2}(t)} + k) I_{N} + Q_{N}^{-1})^{-1}\) is given by the function \(\sigma(\lambda) = \frac{\lambda \gamma^{2}(t)}{\lambda (\beta^{2}(t) + k \gamma^{2}(t)) + \gamma^{2}(t)}\) evaluated over the spectrum of \(Q_{N}\). This function is increasing for \(\lambda \geq 0\), attaining its supremum at \(\frac{\gamma^{2}(t)}{\beta^{2}(t) + k \gamma^{2}(t)}\). Hence,
  \[
    \onorm{C_{N}} \leq \frac{\gamma^{2}(t)}{\beta^{2}(t) + k \gamma^{2}(t)}.
  \] Substituting this  into \Cref{eqn:dxmt},
  \begin{align*}
    \norm{D_{x} m_{1 \mid 0, t, N}(x_{0}, x)[h]}_{H_{C}} &\leq \frac{\beta(t)}{\gamma^{2}(t)} \onorm{C_{N}} \onorm{\Pi_{N}} \norm{h}_{H_{C}} \leq \frac{\beta(t)}{\beta^{2}(t) + k \gamma^{2}(t)} \norm{h}_{H_{C}}.
  \end{align*}
  It follows from the mean-value inequality \citep[][Theorem 2.1.19]{berger1977nonlinearity}, that for any \(x, y \in H\),
  \begin{align}
    \norm{m_{1 \mid 0, t, N}(x_{0}, x) - m_{1 \mid 0, t, N}(x_{0}, y)}_{H_{C}} &= \norm{m_{1 \mid 0, t, N}(x_{0}, x) - m_{1 \mid 0, t, N}(x_{0}, y)}_{H_{N}} \notag \\
    &\leq \frac{\beta(t)}{\beta^{2}(t) + k \gamma^{2}(t)} \norm{x - y}_{H_{C}}. \label{eqn:ineq}
  \end{align}%
  Passing \(N \to \infty\), the sequence of approximate posterior means \(m_{1 \mid 0, t, N}(x_{0}, x)\) converges to the true posterior mean \(m_{1 \mid 0, N}(x_{0}, x)\) (see Lemma \ref{lem:posteriormeanconvergence}). Since each approximation satisfies the inequality (\ref{eqn:ineq}) that is uniform in \(N\) and the norm is a continuous mapping, the true posterior mean \(m_{1 \mid 0, t}(x_{0}, x)\) also inherits the inequality.
  \[
    \norm{m_{1 \mid 0, t}(x_{0}, x) - m_{1 \mid 0, t}(x_{0}, y)}_{H_{C}} \leq \frac{\beta(t)}{\beta^{2}(t) + k \gamma^{2}(t)} \norm{x - y}_{H_{C}}.
  \]

  \paragraph{Step 6} We now substitute this relationship into the expression for the drift coefficient in \Cref{eqn:driftreexpressed}: a Lipschitz constant for the overall drift is the maximum of the Lipschitz constants for each term involving \(x_{t}\):
  \[
    \norm{f(t, x_{0}, x) - f(t, x_{0}, y)}_{H_{C}} \leq L(t) \norm{x - y}_{H_{C}},
  \]
  where
  \[
    L(t) =  \mathop{\underset{}{\max}}\qty{\abs{\frac{\dot{\gamma}(t)}{\gamma(t)} - \frac{\varepsilon}{\gamma^{2}(t)}}, \abs{\dot{\beta}(t) - \beta\qty(\frac{\dot{\gamma}(t)}{\gamma(t)} - \frac{\varepsilon}{\gamma^{2}(t)})} \frac{\beta(t)}{\beta^{2}(t) + k \gamma^{2}(t)}}.
  \]
  This concludes the proof.%

We now state two intermediate results that were used in the above proof.

\begin{lemma}\label{lem:frechetf}
  Let \(m_{1 \mid 0, t, N}(x_{0}, x)\) be an approximate posterior mean as defined in \Cref{eqn:apm}, with \(t \in(0,1)\) and \(N \geq 0\). Then the Fréchet derivative of the mapping \(x \mapsto m_{1 \mid 0, t, N}(x_{0}, x)\) in \(H_{C}\)-norm, in a direction \(h \in H_{C}\) is given by
  \[\resizebox{\linewidth}{!}{\(\displaystyle
        D_{x} m_{1 \mid 0, t, N}(x_{0}, x)[h] =\frac{\beta(t)}{\gamma^{2}(t)} \mathop{\mathbb{E}_{\mu_{1 \mid 0, t, N}(\cdot \mid x_{0}, x)}}\qty[(x_{1, N} - m_{1 \mid 0, t, N}(x_{0}, x)) \ev{x_{1, N} - m_{1 \mid 0, t, N}(x_{0}, x), h}_{H_{C}}].
  \)}\]

  \begin{proof}
    We begin by taking the Fréchet derivative of \(m_{1 \mid 0, t, N}(x_{0}, x)\) at \(x\) in a direction \(h \in H_{C}\). Applying the quotient rule \citep[Chapter 2.1]{berger1977nonlinearity} and simplifying, we have
    \begin{align}
      D_{x} m_{1 \mid 0, t, N}(x_{0}, x)[h] &= D_{x}\qty{\frac{\int_{H_{N}} x_{1, N} \exp(-V_{1 \mid 0, t, N}(x_{1, N},x_{0}, x)) \mathbb{P}_{1 \mid 0, N}(\dd{x_{1, N}}, x_{0})}{ Z_{1 \mid 0, t, N}(x_{0}, x)}}[h] \notag \\
      &= \frac{1}{Z_{1 \mid 0, t, N}(x_{0}, x)} D_{x}U_{1 \mid 0, t, N}(x_{0}, x)[h] - m_{1 \mid 0, t, N}(x_{0}, x) \frac{D_{x} Z_{1 \mid 0, t, N}(x_{0}, x)[h]}{Z_{1 \mid 0, t, N}(x_{0}, x)}, \label{eqn:lem10fin}
    \end{align}
    where we define \(U_{1 \mid 0, t, N}(x _{0}, x) \coloneqq \int_{H_{N}} x_{1, N} \exp(-V_{1 \mid 0, t, N}(x_{1, N},x_{0}, x)) \mathbb{P}_{1 \mid 0, N}(\dd{x_{1, N}}, x_{0})\) to simplify notation. Evaluating the Fréchet derivatives, we have
    \begin{align*}
      D_{x}U_{1 \mid 0, t, N}(x_{0}, x)[h] &= \frac{1}{\gamma^{2}(t)}\int_{H_{N}} x_{1, N} \ev{\alpha(t) \Pi_{N} x_{0} + \beta(t) x_{1, N} - x, h}_{H_{C}}\\[-0.5em]
      &\qquad \qquad \cdot \exp(-V_{1 \mid 0, t, N}(x_{1, N}, x_{0}, x)) \mathbb{P}_{1 \mid 0, N}(\dd{x_{1, N}}, x_{0}), \\
      D_{x} Z_{1 \mid 0, t, N}(x_{0}, x)[h] &= \frac{1}{\gamma^{2}(t)}\int_{H_{N}} \ev{\alpha(t) \Pi_{N} x_{0} + \beta(t) x_{1, N} - x, h}_{H_{C}} \\[-0.5em]
      &\qquad \qquad \cdot \exp(-V_{1 \mid 0, t, N}(x_{1, N}, x_{0}, x)) \mathbb{P}_{1 \mid 0, N}(\dd{x_{1, N}}, x_{0}).
    \end{align*}
    Substituting these into \Cref{eqn:lem10fin} and recognising that the fractions come together to form the approximate posterior density, we have:
    \[\resizebox{\linewidth}{!}{\(\displaystyle
          D_{x} m_{1 \mid 0, t, N}(x_{0}, x)[h] = \frac{1}{\gamma^{2}(t)}\mathop{\mathbb{E}_{\mu_{1 \mid 0, t, N}(\cdot \mid x_{0}, x)}}\qty[(x_{1, N} - m_{1 \mid 0, t, N}(x_{0}, x))\ev{\alpha(t) \Pi_{N} x_{0} + \beta(t) x_{1, N} - x, h}_{H_{C}}].
    \)}\]
    Adding and subtracting zero,
    \[\resizebox{\linewidth}{!}{\(\displaystyle
          0 = \frac{1}{\gamma^{2}(t)} \mathop{\mathbb{E}_{\mu_{1 \mid 0, t ,N}(\cdot \mid x_{0}, x)}}\qty[(x_{1, N} - m_{1 \mid 0, t, N}(x_{0}, x)) \ev{- \alpha(t) \Pi_{N} x_{0}  + \beta(t) m_{1 \mid 0, t, N}(x_{0}, x) + x, h}_{H_{C}}],
    \)}\]
    we arrive at
    \[\resizebox{\linewidth}{!}{\(\displaystyle
          D_{x} m_{1 \mid 0, t, N}(x_{0}, x)[h] = \frac{\beta(t)}{\gamma^{2}(t)}\mathop{\mathbb{E}_{\mu_{1 \mid 0, t, N}(\cdot \mid x_{0}, x)}}\qty[(x_{1, N} - m_{1 \mid 0, t, N}(x_{0}, x))\ev{x_{1, N} - m_{1 \mid 0, t, N}(x_{0}, x), h}_{H_{C}}].
    \)}\]
    This concludes the proof.
  \end{proof}
\end{lemma}

\begin{lemma}\label{lem:posteriormeanconvergence}
  For every \(x_{0}, x \in H\) and \(t \in (0, 1)\), the sequence of approximate posterior means \(\qty{m_{1 \mid 0, t, N}(x_{0}, x)}_{N=1}^{\infty}\)   as defined in \Cref{eqn:apm} converges to the true posterior mean \(m_{1 \mid 0, t}(x_{0}, x)\).

  \begin{proof}
    First, let us re-express the definition of \(m_{1 \mid 0, t, N}(x_{0}, x)\) by lifting the integrals into a common infinite-dimensional space:
    \begin{align}
      m_{1 \mid 0, t, N}( x_{0}, x) &= \int_{H_{C}} \Pi_{N} x_{1} \frac{1}{Z_{1 \mid 0, t, N}(x_{0}, x)} \exp(-V_{1 \mid 0, t}(\Pi_{N} x_{1}, \Pi_{N} x_{0}, x)) \mathbb{P}_{1 \mid 0}(\dd{x_{1}}, x_{0}),  \label{eqn:lift}\\
      \text{ where } Z_{1 \mid 0, t, N}(x_{0}, x) &= \int_{H_{C}} V_{1 \mid 0, t}(\Pi_{N} x_{1}, \Pi_{N} x_{0}, x) \mathbb{P}_{1 \mid 0}(\dd{x_{1}}, x_{0}).\notag
    \end{align}

    We define the sequence of functions \[f_{N}(x_{1}) \coloneqq \Pi_{N} x_{1} \frac{1}{Z_{1 \mid 0, t, N}(x_{0}, x)}\exp(- V_{1 \mid 0, t}(\Pi_{N} x_{1}, \Pi_{N} x_{0}, x)),\] and \[f(x_{1}) \coloneqq x \frac{1}{Z_{1 \mid 0, t}(x_{0}, x)}\exp(-V_{1 \mid 0, t}(x_{1}, x_{0}, x)),\] for fixed \(x_{0}\) and  \(x\). To show convergence, we appeal to the Vitali convergence theorem \citep{walnut2011vitali}, which  states that if the sequence of functions \(f_{N}\) is pointwise-convergent to \(f\) and uniformly integrable, then the integral of the functions also converges to the integral of \(f\). We proceed in two steps.

    \paragraph{Step 1: Pointwise Convergence} The numerator \(\Pi_{N} x_{1} \exp(-V_{1 \mid 0, t}(\Pi_{N} x_{1}, \Pi_{N} x_{0}, x))\) is clearly pointwise convergent to \(x_{1} \exp(-V_{1 \mid 0, t}(x_{1}, x_{0}, x))\) since for any fixed \(x_{1} \in H_{C}\), the projection \(\Pi_{N} x_{1}\) converges to \(x_{1}\) in \(H_{C}\)-norm, and \(V_{1 \mid 0, t, x}\) is continuous in all of its inputs. Hence, it remains to show convergence of the sequence of normalising constants \(Z_{1 \mid 0, t, N}(x_{0}, x)\).

    To this end, we apply the dominated convergence theorem to show that
    \[
      \lim\limits_{N \to \infty} \int_{H_{C}} \exp(-V_{1 \mid 0}(\Pi_{N} x_{1}, \Pi_{N} x_{0}, x)) \mathbb{P}_{1 \mid 0}(x_{0}) = \lim\limits_{N \to \infty} \int_{H_{C}} \exp(-V_{1 \mid 0}(x_{1}, x_{0}, x)) \mathbb{P}_{1 \mid 0}(\dd{x_{1}}, x_{0}).
    \]
    Since \(\Phi\) is strongly convex, the integrand is uniformly bounded. By the dominated convergence theorem, \(\lim\limits_{N \to \infty} Z_{1 \mid 0, t, N}(x_{0}, x) = Z_{1 \mid 0, t}(x_{0}, x)\).

    Finally, since the normalising constant is nonzero for any \(N\) and converges to a non-zero value, the functions \(f_{N}(x_{1})\) are pointwise convergent to \(f(x_{1})\).

    \paragraph{Step 2: Uniform Integrability}
    A sufficient condition for uniform integrability is that there exists a uniform bound on the expected squared norm of sequence of the functions \(f_{N}\) \citep[][Theorem 3.5]{billingsley2013convergence}:
    \begin{equation}
      \int_{H_{C}} \norm{\Pi_{N} x_{1}}_{H_{C}}^{2} \frac{1}{Z_{1 \mid 0, t, N}^{2}(x_{0}, x)} \exp(-2 V_{1 \mid 0, t}(\Pi_{N} x_{1}, \Pi_{N} x_{0}, x)) \mathbb{P}_{1 \mid0}(\dd{x_{1}}, x_{0}). \label{eqn:dc2}
    \end{equation}
    Since the squared normalising factors \(Z^{2}_{1 \mid 0, t, N}(x_{0}, x)\) are positive for all \(N\) and converge to a positive value and  \(\Phi\) is strongly convex, the dominated convergence theorem  applies again, and it follows that the sequence of integrals in \Cref{eqn:dc2} is convergent and therefore bounded. Hence, the sequence of functions \(f_{N}\) is uniformly integrable.

    Since the sequence of functions \(f_{N}\) is pointwise convergent and uniformly integrable, it follows that their integrals, by definition equal to the approximate posterior means \(m_{1 \mid 0, t, N}(x_{0}, x)\), converge to the true posterior mean \(m_{1 \mid 0, t}(x_{0}, x)\).
  \end{proof}
\end{lemma}

  %
  %
  %
  %
  %
\end{proof}

\subsection{Proof of \Cref{lem:lipchitz-conditions}\ref{setting:manifold}} \label{prf:prp:manifold}

\paragraph{Proposition \ref{lem:lipchitz-conditions}\ref{setting:manifold}.}
{\itshape
    Suppose the law \(\mu_{1}\) of the target data \(x_{1}\) is supported on a bounded subset of \(H_{C}\), that is, there exists a scalar \(R < \infty\) where \(\norm{x_{1}}_{H_{C}} < R\), \(\mu_{1}\)-almost surely. Then the map \(x \mapsto f(t, x_{0}, x)\) is Lipschitz continuous with respect to the \(H_{C}\)-norm. Specifically, for each \(t \in (0, 1)\) and \(x_{0}, x \in H\), the following inequality holds:
    \[
      \norm{f(t, x_{0}, x) - f(t, x_{0}, y)}_{H_{C}} \leq L(t) \norm{x - y}_{H_{C}},
    \]
    where the Lipschitz constant \(L(t)\) is:
    \[
      L(t) =  \mathop{\underset{}{\max}}\qty{\abs{\frac{\dot{\gamma}(t)}{\gamma(t)} - \frac{\varepsilon}{\gamma^{2}(t)}}, \abs{\dot{\beta}(t) - \beta(t)\qty(\frac{\dot{\gamma}(t)}{\gamma(t)} - \frac{\varepsilon}{\gamma^{2}(t)})} \frac{R^{2} \beta(t)}{\gamma^{2}(t)}}.
    \]
}
\begin{proof}
  We follow a similar overarching argument to \Cref{prf:prp:bayes}. Again, the expression \Cref{eqn:driftreexpressed} means it is sufficient to consider Lipschitz-continuity of \(x\mapsto \mathop{\mathbb{E}}\qty[x_{1} \mid x_{0}, x_{t} = x]\). Bounded support in \(H_{C}\)-norm allows us to greatly simplify our arguments, avoiding a Galerkin-type projection argument and directly providing our proof in infinite dimensions.


  As in \Cref{prf:prp:bayes}, we let \(\mu_{1 \mid 0, t}(\cdot \mid x_{0}, x)\) denote the posterior law of \(x_{1}\), conditional on \(x_{0}\) and \(x_{t} = x\). This time however, for each \(t \in (0, 1)\) we let the reference measure be \(\operatorname{N}(\alpha(t) x_{0}, \gamma^{2}(t) C)\). Note that the Cameron-Martin space of \(\gamma^{2}(t)C\) is identical to that of \(C\), equipped with an inner product scaled by \(\frac{1}{\gamma^{2}(t)}\). Since \(\beta(t) x_{1}\) is almost-surely in \(H_{C}\), and hence also the Cameron-Martin space of \(\gamma^{2}(t)C\), \(H_{\gamma^{2}(t)C}\), the measure \(\mu_{1 \mid 0, t}(\cdot \mid x_{0}, x)\) is absolutely continuous with respect to the reference measure:
  \begin{align*}
    \dv{\mu_{1 \mid 0, t}(\cdot \mid x_{0}, x)}{\operatorname{N}(\alpha(t) x_{0}, \gamma^{2}(t) C)}{}(x_{1}) &= \frac{1}{Z_{1 \mid 0, t}(x_{0}, x)}\exp(-V_{1 \mid 0, t}(x_{1}, x_{0}, x)), \\
    \text{ where } V_{1 \mid 0, t}(x_{1}, x_{0}, x) &= \frac{1}{\gamma^{2}(t)} \norm{\alpha(t) x_{0} + \beta(t) x_{1} - x}_{H_{C}}^{2},
  \end{align*}
  and \(Z_{1 \mid 0, t}(x_{0}, x) \coloneqq \int_{H_{C}} \exp(-V_{1 \mid 0, t}(x_{1}, x_{0}, x)) \operatorname{N}(\dd{x_{1}}; \alpha(t) x_{0}, \gamma^{2}(t) C)\) is a normalising constant. We define \(m_{t}(x_{0}, x)\) as the posterior mean:
  \[
    m_{t}(x_{0}, x) \coloneqq \mathop{\mathbb{E}_{\mu_{1 \mid 0, t}(\cdot \mid x_{0}, x)}}\qty[x_{1}] = \int_{H_{C}} x_{1} \mu_{1 \mid 0, t}(\dd{x_{1}} \mid x_{0}, x).
  \]
  Following an approach analogous to that given in the proof to \Cref{lem:frechetf}, we take the Fréchet derivative in the direction \(h \in H_{C}\) and again arrive at a covariance:
  \[
    D_{x} m_{t}(x_{0}, x)[h] = \frac{\beta(t)}{\gamma^{2}(t)} \mathop{\mathbb{E}_{\mu_{1 \mid 0, t}(\cdot \mid x_{0}, x)}}\qty[ (x_{1} - m_{t}(x_{0}, x))\ev{x_{1} - m_{t}(x_{0}, x), h}_{H_{C}}]
  \]
  Taking the norm in \(H_{C}\) and applying the Cauchy-Schwarz inequality, we have
  \[
    \norm{D_{x} m_{t}(x_{0}, x)[h]}_{H_{C}} \leq \frac{\beta(t)}{\gamma^{2}(t)} \mathop{\mathbb{E}_{\mu_{1 \mid 0, t}(\cdot \mid x_{0}, x)}}\qty[\norm{x_{1} - m_{t}(x_{0}, x)}_{H_{C}}^{2}] \norm{h}_{H_{C}}
  \]
  Using the fact that \(0 \leq \mathop{\mathbb{E}}\qty[\norm{x_{1} - m_{t}(x_{0}, x)}_{H_{C}}^{2}] = \mathop{\mathbb{E}}\qty[\norm{x_{1}}_{H_{C}}^{2}] - \norm{m_{t}(x_{0}, x)}^{2}\) and \(\norm{x_{1}^{2}}_{H_{C}} \leq R^{2}\) almost surely, we conclude
  \[
    \norm{D_{x} m_{t}(x_{0}, x)[h]}_{H_{C}} \leq \frac{R^{2} \beta(t)}{\gamma^{2}(t)} \norm{h}_{H_{C}}.
  \]
  Finally, we apply the mean-value inequality \citep[][Theorem 2.1.19]{berger1977nonlinearity} and conclude that \(m_{t}(x_{0}, x)\) is Lipschitz in \(H_{C}\)-norm with Lipschitz constant at most \(\frac{R^{2}\beta(t)}{\gamma^{2}(t)}\):
  \[
    \norm{m_{t}(x_{0}, x) - m_{t}(x_{0}, y)}_{H_{C}} \leq \frac{R^{2}\beta(t)}{\gamma^{2}(t)} \norm{x - y}_{H_{C}}.
  \]
  Substituting this into \Cref{eqn:driftreexpressed} gives the Lipschitz constant for the overall mapping \(x \mapsto f(t, x_{0}, x)\). This concludes the proof.
\end{proof}

  \begin{remark}\label{rem:hc}
    Both settings \ref{setting:bayes} and \ref{setting:manifold} involve the essential assumption that  data \(x_{1}\) is supported on the Cameron-Martin space \(H_{C}\). This ensures that the measures corresponding to the process under different data realizations are equivalent, not mutually singular (by the Cameron-Martin and Feldman-Hajek theorems; \citealp{bogachev1998gaussian,stuart2010inverse}).

    Intuitively, the restriction of \(x_{1}\) to \(H_{C}\) is a smoothness assumption that confines realisations of \(x_{1}\) to a class of functions that are fundamentally smoother than typical realisations of the noise \(\gamma^{2}(t) z\). This ensures that the Gaussian measures corresponding to translations of scaled-noise \(\gamma^{2}(t)\) by different candidates \(x_{1}', x_{1}''\) are always equivalent, allowing for an expression of the posterior measure as a well-defined density with respect to some reference measure. Otherwise, there would exist no reference measure with respect to which the posterior measure has a density, complicating the analysis of Lipschitz continuity.
  \end{remark}

\subsection{Proof of \Cref{thm:exist}} \label{prf:thm:exist}

\restatethmexist*

\begin{proof}
  We begin by addressing the behaviour of the drift and its associated Lipschitz constant, \(L(t)\), at the initial time \(t = 0\). The drift coefficient \(f(t, x_{0}, x)\) is defined via conditional expectations of the stochastic interpolant \(x_{t}\), conditioned on \(x_{0}\) and \(x_{t} = x\).

  At the specific instant \(t = 0\), the conditioning events are only consistent if \(x = x_{0}\). Consequently, at time \(0\), the drift \(f(0, x_{0}, x)\) is only well-defined where \(x_{0} = x\). This is satisfied by the initial condition \(X_{0} = x_{0}\) for the CB-SDE (\ref{eqn:cbsde}). However, the Lipschitz condition is a statement about the behaviour of the drift under perturbations, i.e., comparing \(f(t, x_{0}, x)\) and \(f(t, x_{0}, y)\) for \(x \neq y\). Since the drift is not defined for such perturbations at \(t = 0\), the Lipschitz condition is only meaningful for \(t > 0\). Therefore, for the purpose the arguments below, we  extend the function \(L(t)\) to be continuous on the entire closed interval \([0, \overline{t}]\). Without loss of generality, we define \(L(0) := \lim_{t \to 0^{+}} L(t)\). This is justified because the value of the drift at a single point in time does not affect the SDE's solution.
  \paragraph{Step 1: Partitioning of Time Domain}

  With the above remark, we proceed with the proof assuming that \(L(t)\) is continuous and therefore bounded on the compact interval \([0, \overline{t}]\). Hence, it is possible to create a finite partition \(0 = \tau_{0} < \tau_{1} < \tau_{2} < \cdots < \tau_{k} < \cdots < \tau_{K} = {\overline{t}}\) of \([0, {\overline{t}}]\) with \(K < \infty\) such that
  \[
    q_{k} \coloneqq (\tau_{k} - \tau_{k-1}) \sup_{t \in [\tau_{k-1}, \tau_{k}]} {L}(t) < 1, \quad \text{ for all } k = 1, \ldots, K.
  \]

  \paragraph{Step 2: Existence of Strong Solutions}

  For each \(k = 1, \ldots, K\), consider the Banach space \(B_{k}\) of all continuous, \(H_{C}\)-valued functions on \([\tau_{k-1}, \tau_{k}]\) equipped with the following norm:
  \[\norm{Y}_{B_{k}} \coloneqq \sup_{t \in [\tau_{k-1}, \tau_{k}]} \norm{Y(t)}_{H_{C}}.\]
  To argue existence of a strong solution to the CB-SDE on \([0, {\overline{t}}]\), we will apply Banach's fixed point theorem inductively and piecewise on the intervals \(\qty[\tau_{k-1}, \tau_{k}]\) and pathwise for all events \(\omega\) in the sample space \(\Omega\), to build a solution \({X}_{t}\) on \([0, \overline{t}]\).

  Fix any event \(\omega \in \Omega\), so that \(x_{0}(\omega)\) and \({W}_{t}(\omega)\) are respectively the outcomes of the random variable \(x_{0}\) and the Wiener process at time \(t\), and define \({X}_{0}(\omega) \coloneqq x_{0}(\omega)\). Furthermore, let
  \[
    \widetilde{W}_{k, t} \coloneqq \int_{\tau_{k-1}}^{t} \sqrt{2\varepsilon }\dd{ {W}_{s}}.
  \] We proceed by induction: for each \(k = 1, \ldots, K\), having defined \({X}_{\tau_{k-1}}(\omega)\), we define the mapping \(\Psi_{k, \omega} : B_{k} \to B_{k}\) as follows. For any \(Y \in B_{k}\),
  \begin{equation}
    (\Psi_{k, \omega} Y)(t) = \int_{\tau_{k-1}}^{t} {f}\qty(s, x_{0}(\omega), {X}_{\tau_{k-1}}(\omega) + \widetilde{W}_{k, s}(\omega) + Y(s)) \dd{s}, \quad \text{ for all } t \in [\tau_{k-1}, \tau_{k}]. \label{eqn:banach-iteration}
  \end{equation}

  For any \(Y, Y' \in B_{k}\), we have
  \par\noindent
  \resizebox{\linewidth}{!}{
    \begin{minipage}{\linewidth}
      \begin{align*}
        &\norm{\Psi_{k, \omega} Y - \Psi_{k, \omega} Y'}_{B_{k}} = \sup_{t \in [\tau_{k-1}, \tau]} \norm{ (\Psi_{k, \omega} Y - \Psi_{k, \omega} Y')(t) }_{H_{C}} \\
        &\qquad\leq \int_{\tau_{k-1}}^{\tau_{k}} \norm{ {f}\qty(s, x_{0}(\omega), {X}_{\tau_{k-1}}(\omega) + \widetilde{W}_{k, s}(\omega) + Y(s)) - {f}\qty(s, x_{0}(\omega), {X}_{\tau_{k-1}}(\omega) + \widetilde{W}_{k, s}(\omega) + Y'(s))}_{H_{C}} \dd{s}\\
        &\qquad\leq (\tau_{k} - \tau_{k-1})\sup_{t \in [\tau_{k-1}, \tau]} \qty[{L}(t) \norm{Y(t) - Y'(t)}_{H_{C}}] \\
        &\qquad\leq (\tau_{k} - \tau_{k-1})\sup_{t \in [\tau_{k-1}, \tau]} {L}(t) \sup_{t \in [\tau_{k-1}, \tau]} \norm{Y(t) - Y'(t)}_{H_{C}} \\
        &\qquad= q_{k} \norm{Y - Y'}_{B_{k}},\\
      \end{align*}
  \end{minipage}}
  \par\noindent
  where \(q_{k} < 1\) by construction of the interval. By Banach's fixed point theorem, it follows that there exists a unique \(Y^{*} \in B_{k}\) such that \(\Psi_{k, \omega} Y^{*} = Y^{*}\).

  For every \(t \in [\tau_{k-1}, \tau_{k}]\), we let \({X}_{t}(\omega) \coloneqq {X}_{\tau_{k-1}}(\omega) +  \widetilde{W}_{k, t}(\omega) + Y^{*}(t)\) for all \(t \in [\tau_{k-1}, \tau_{k}]\). Substituting this definition into the fixed point identity \(\Psi_{k, \omega} Y^{*} = Y^{*}\), we have
  \begin{align*}
    {X}_{t}(\omega) - {X}_{\tau_{k-1}}(\omega) - \widetilde{W}_{k, t}(\omega) &= \int_{\tau_{k-1}}^{t} {f}(s, x_{0}(\omega), {X}_{s}(\omega)) \dd{s} \\
    \implies {X}_{t}(\omega) &= {X}_{\tau_{k-1}}(\omega) +  \int_{\tau_{k-1}}^{\tau_{k}} {f}(s, x_{0}(\omega) {X}_{s}(\omega)) \dd{s} + \int_{\tau_{k-1}}^{t} \sqrt{2\varepsilon } \dd{ {W}(\omega)},
  \end{align*}
  which is the integral form of the CB-SDE (\ref{eqn:cbsde}), expressed pathwise with the chosen probability event \(\omega \in \Omega\) and defined on the interval \(t \in [\tau_{k-1}, \tau_{k}]\).

  Since \(\omega\) was chosen arbitrarily, we may repeat this process for every \(\omega \in \Omega\) to build a stochastic process \({X}_{t}\) on the interval \(t \in [\tau_{k-1}, \tau_{k}]\). Now that we have a definition of \({X}_{\tau_{k}}(\omega)\), we may repeat the inductive step for \(k \leftarrow k+1\). This builds a stochastic process \({X}_{t}\) on the entire desired interval \(t \in [ 0, {\overline{t}}]\).

  It remains to check that \({X}_{t}\) is \({\mathbb{F}}\)-adapted on \([0, {\overline{t}}]\). Again, employing induction, we may observe that \(X_{0} = x_{0}\) is by definition \(\mathcal{F}_{0}\)-measurable. Then, for each \(k = 1, \ldots, K\), we are given that \(X_{\tau_{k-1}}\) is \(\mathcal{F}_{ \tau_{k-1}}\)-measurable. We can view every contraction-mapping iteration as if it were applied for all \(\omega \in \Omega\) simultaneously. Suppose the initial guesses \(Y_{\omega} \in B_{k}\) are such that \(Y_{\omega}(t)\) is \(\mathcal{F}_{t}\)-measurable as a function of \(\omega\), for all \(t \in [\tau_{k-1}, \tau_{k}]\). Each application of the contraction mapping, \((\Psi_{k,\omega}(Y_{\omega}))(t)\), is also \(\mathcal{F}_{t}\)-measurable as a function of \(\omega\), since the integrand in \Cref{eqn:banach-iteration} is the composition of a continuous function with an \(\mathcal{F}_{t}\)-measurable function. Hence, every time we perform a Banach iteration, the outcome at time \(t \in [\tau_{k-1}, \tau]\) is \(\mathcal{F}_{t}\)-measurable. Since \(\sigma\)-fields are closed under countable pointwise limits, it follows that \(Y^{*}_{\omega}(t)\) and thus \({X}_{t}(\omega)\) are \(\mathcal{F}_{t}\)-measurable for all \(t \in [\tau_{k-1}, \tau_{k}]\). Repeating the induction for all steps up to \(k = K\) ensures that \({X}_{t}(\omega)\) is \(\mathcal{F}^{*}_{t}\)-measurable for all \(t \in [0, {\overline{t}}]\) and hence \({X}_{t}\) is \({\mathbb{F}}\)-adapted on \([0, {\overline{t}}]\). This concludes the proof.
\end{proof}

\subsection{Proof of \Cref{thm:uniq}} \label{prf:thm:uniq}
\paragraph{Theorem \ref{thm:uniq}.} {
\itshape

    Assume the same Lipschitz conditions of \Cref{thm:exist}.

    Let \(\qty{e_{n}}_{n=1}^{\infty}\) be an orthonormal basis of eigenvectors for the covariance operator \(C\), and let \(H_{N}\) be the subspace of \(H_{C}\) spanned by \(\qty{e_{1}, \ldots, e_{N}}\). We denote by \(P_{N}\) the orthogonal projection operator from \(H\) into \(H_{N}\).

    Suppose the distribution \(\mu_{1}\) of target data \(x_{1}\) is such that the projections \(\ev{x_{1}, e_{n}}\) are mutually independent random variables for different indices \(n\). Then, the solution to the CB-SDE (\ref{eqn:cbsde}) is unique.
}

\begin{proof}
  As in the proof in \Cref{prf:thm:exist} for the existence of strong solutions, we assume without loss of generality that \(L(t)\) is continuous on \([0, \overline{t}]\). Let \({X}_{t}\) and \({X}_{t}'\) be two strong solutions for the same initial condition, \({X}_{0} = {X}_{0}' = x_{0}\) and driven by the same Wiener process \({W}_{t}\) on \([0, {\overline{t}}]\). \textit{A priori}, it is not guaranteed that \(\norm{ {X}_{t} - {X}_{t}'}_{H_{C}} < \infty\) since \({X}_{t} - {X}_{t}'\) may not be in \(H_{C}\). However, for each \(N \geq 1\), it is guaranteed that the projected difference \(P_{N}( {X}_{t} - {X}_{t}') \in H_{C}\) since the range of \(P_{N}\) is by definition a subspace of \(H_{C}\) due to \(C\) being a positive-definite operator. It therefore holds that
  \begin{align*}
    \dv{t}P_{N}\qty({X}_{t} - {X}_{t}') &= P_{N}\qty({f}(t, x_{0}, {X}_{t}) - {f}(t, x_{0}, {X}_{t}')) \\
    \implies \dv{t} \norm{ P_{N} \qty({X}_{t} - {X}_{t}')}_{H_{C}} &\leq \norm{P_{N}\qty({f}(t, x_{0}, {X}_{t}) - {f}(t, x_{0}, {X}_{t}'))}_{H_{C}} \\
    &\leq {L}(t)\norm{P_{N}\qty( {X}_{t} - {X}_{t}')}_{H_{C}}.
  \end{align*}

  Since \({L}(t)\) is real-valued and continuous on \([0, {\overline{t}}]\), we may now apply Grönwall's inequality \citep[][Theorem 1.2.2]{ames1997inequalities} to the quantity \(\norm{P_{N}\qty({X}_{t} - {X}_{t}')}_{H_{C}}\) as a function of \(t\):
  \[
    \norm{P_{N}\qty( {X}_{t} - {X}_{t}')}_{H_{C}} \leq \norm{ P_{N}\qty( {X}_{0} - {X}_{0}')}_{H_{C}} \exp(\int_{0}^{ {\overline{t}}} {L}(t) \dd{t} ).
  \]
  Since by definition \({X}_{0} = {X}_{0}' = x_{0}\), so \({X}_{0} - {X}_{0}' = 0\), it follows that
  \[
    \norm{P_{N}\qty({X}_{t} - {X}_{t}')}_{H_{C}} = 0,
  \]
  for all \(t \in [0, {\overline{t}}]\). Since this equality is true for every \(N \geq 1\), we pass \(N \to \infty\). It follows that \(\norm{ {X}_{t} - {X}_{t}'}_{H_{C}} = 0\) and therefore
  \[
    {X}_{t} = {X}_{t}', \text{ for all } t \in [0, \overline{t}].
  \]
  This concludes the proof.
\end{proof}

\subsection{Proof of \Cref{lem:tc}}\label{prf:lem:tc}
\restatelemtc*

\begin{proof}
  Let \(h(t) \coloneqq \varepsilon/\gamma(t)\) and \(C \coloneqq \int_0^1 h(\sigma) d\sigma\), which is finite and positive by assumption. We define a new time variable \(s(t)\) by
  \[ s(t) \coloneqq \frac{1}{C} \int_{0}^{t} h(\sigma) \dd{\sigma}. \]
  Since \(h(t) > 0\) for \(t \in (0, 1)\), \(s(t)\) is a strictly increasing, continuously differentiable bijection from \([0, 1]\) to itself. We define the time-change \(\theta(t)\) as its inverse, \(\theta(t) \coloneqq s^{-1}(t)\). This is also strictly increasing and continuously differentiable on \((0, 1)\), with derivative \(
  \dot{\theta}(t) = \frac{C}{\varepsilon} \gamma(\theta(t))\).
  Substituting this into the definition of \(\hat{c}(t)\), we have
  \[
    \hat{c}(t) = \frac{C}{\varepsilon} \qty(\dot{\gamma}(\theta(t)) \gamma(\theta(t))) - C.
  \]
  By assumption, the function \(\dot{\gamma}(t)\gamma(t)\) has a continuous extension to \([0, 1]\). Since \(\theta(t)\) is also continuous on \([0, 1]\), their composition \(\dot{\gamma}(\theta(t))\gamma(\theta(t))\) is continuous on \([0, 1]\). Therefore, the final expression for \(\hat{c}(t)\) is continuous on \([0, 1]\). This implies that \(\hat{c}(t)\), initially defined only on \((0, 1)\), has well-defined finite limits as \(t \to 0^{+}\) and \(t \to 1^{-}\), and thus admits a continuous extension to the compact interval \([0, 1]\).
\end{proof}

\subsection{Proof of Proposition \ref{prp:loss}}\label{prf:prp:loss}

\begin{restatable}{proposition}{restateprploss}\label{prp:loss}
Assume $t\sim\mathrm{U}(0,1)$, $w_\varphi,w_\eta\in L^{1}([0,1])$ with $w_\varphi,w_\eta>0$, and $\mathbb{E}\|x_0\|_{H}^{2}+\mathbb{E}\|x_1\|_{H}^{2}<\infty$. Define the (weighted) regression objectives
\begin{align}
\mathcal{E}_\varphi(\theta)
&:=\mathbb{E}\!\left[w_\varphi(t)\,\|\varphi_\theta(x_t,t)-\varphi(x_t,t)\|_{H}^{2}\right],\\
\mathcal{E}_\eta(\phi)
&:=\mathbb{E}\!\left[w_\eta(t)\,\|\eta_\phi(x_t,t)-\eta(x_t,t)\|_{H}^{2}\right].
\end{align}
Then, $\mathcal{L}_\varphi(\theta)-\mathcal{E}_\varphi(\theta)$ and $\mathcal{L}_\eta(\phi)-\mathcal{E}_\eta(\phi)$ are finite constants.
\end{restatable}
\begin{proof}
See \Cref{prf:prp:loss}.
\end{proof}

\begin{proof}
We treat the denoiser and velocity losses separately.

\paragraph{Denoiser loss.}
Let $\{e_n\}_{n=1}^\infty$ be an orthonormal basis of $H$, let $\Pi_N$ denote the orthogonal projection onto $\mathrm{span}\{e_1,\dots,e_N\}$, and define
\[
\mathcal{L}_{\eta,N}(\phi)
:=\mathbb{E}\!\left[w_\eta(t)\,\|\Pi_N(\eta_\phi(x_t,t)-z)\|_H^2\right].
\]
Writing $\eta_\phi(x_t,t)-z=(\eta_\phi(x_t,t)-\eta(x_t,t))+(\eta(x_t,t)-z)$ and expanding the square yields
\[
\mathcal{L}_{\eta,N}(\phi)
=
\mathbb{E}\!\left[w_\eta(t)\,\|\Pi_N(\eta_\phi(x_t,t)-\eta(x_t,t))\|_H^2\right]
+
C_{\eta,N},
\]
since the cross term vanishes by iterated expectation: $\eta_\phi(x_t,t)-\eta(x_t,t)$ is $\sigma(x_t)$-measurable and
$\mathbb{E}[\eta(x_t,t)-z\mid x_t]=0$.
Here
\[
C_{\eta,N}:=\mathbb{E}\!\left[w_\eta(t)\,\|\Pi_N(\eta(x_t,t)-z)\|_H^2\right].
\]
Letting $N\to\infty$ and using monotone convergence gives
\[
\mathcal{L}_\eta(\phi)
=
\mathbb{E}\!\left[w_\eta(t)\,\|\eta_\phi(x_t,t)-\eta(x_t,t)\|_H^2\right]
+
C_\eta
=
\mathcal{E}_\eta(\phi)+C_\eta,
\]
where
\[
C_\eta=\mathbb{E}\!\left[w_\eta(t)\,\|\eta(x_t,t)-z\|_H^2\right]
\le \mathbb{E}[w_\eta(t)]\,\mathbb{E}\|z\|_H^2<\infty,
\]
since $w_\eta\in L^1([0,1])$ and the trace-class noise $z$ has finite second moment in $H$.

\paragraph{Velocity loss.}
An identical argument applies to the velocity objective. Writing
\[
\dot\alpha(t)x_0+\dot\beta(t)x_1-\varphi_\theta(x_t,t)
=
(\varphi(x_t,t)-\varphi_\theta(x_t,t))
+
(\dot\alpha(t)x_0+\dot\beta(t)x_1-\varphi(x_t,t)),
\]
projecting onto $\mathrm{span}\{e_1,\dots,e_N\}$, and expanding the squared norm yields
\[
\mathcal{L}_{\varphi,N}(\theta)
=
\mathbb{E}\!\left[w_\varphi(t)\,\|\Pi_N(\varphi_\theta(x_t,t)-\varphi(x_t,t))\|_H^2\right]
+
C_{\varphi,N},
\]
where the cross term again vanishes since
$\mathbb{E}[\dot\alpha(t)x_0+\dot\beta(t)x_1-\varphi(x_t,t)\mid x_t]=0$.
Passing to the limit $N\to\infty$ gives
\[
\mathcal{L}_\varphi(\theta)=\mathcal{E}_\varphi(\theta)+C_\varphi,
\qquad
C_\varphi=\mathbb{E}\!\left[w_\varphi(t)\,\|\dot\alpha(t)x_0+\dot\beta(t)x_1-\varphi(x_t,t)\|_H^2\right].
\]
Finally, since $\alpha,\beta\in C^1([0,1])$, $\dot\alpha,\dot\beta$ are bounded, and hence
\[
C_\varphi
\le \mathbb{E}\!\left[w_\varphi(t)\,\|\dot\alpha(t)x_0+\dot\beta(t)x_1\|_H^2\right]
\le 2\,\mathbb{E}[w_\varphi(t)]
\big(\|\dot\alpha\|_\infty^2\,\mathbb{E}\|x_0\|_H^2
+
\|\dot\beta\|_\infty^2\,\mathbb{E}\|x_1\|_H^2\big)
<\infty.
\]
This concludes the proof.
\end{proof}

\subsection{Proof of Theorem \ref{thm:w2}}\label{prf:thm:w2}
\restatetheoremw*
\begin{proof}
  Let \(\widetilde{X}_{t}\) be the solution to the CB-SDE when using the approximate velocity \(\widetilde{\varphi}\) and denoiser \(\widetilde{\eta}\) to form the approximate drift \(\widetilde{f}(t, x_{0}, x) \coloneqq \widetilde{\varphi}(t, x_{0}, x) + c(t) \widetilde{\eta}(t, x_{0}, x)\). From \Cref{thm:cbsde}, we know that the law of \(X_{t}\) is equal to \(\mu_{t \mid 0}(\dd{x_{t}}, x_0)\). Hence,  we couple \(\widetilde{X}_{t}\) with \(X_{t}\) via the same \(C\)-Wiener process \(W_{t}\) and analyse the expected squared distance between these processes.

  We consider the TC-CB-SDE (\cref{eqn:tccbsde}), which has a unique solution \(\hat{X}_{t} = X_{\theta(t)}\) on the interval \([0, \theta^{-1}(\overline{t})]\) when driven by the Wiener process \(\hat{W}_{t} = W_{\theta(t)}\). Let \(\hat{\widetilde{X}}_{t} = \widetilde{X}_{\theta(t)}\) be the time-changed approximate counterpart. Applying Ito's lemma (\citealp[Theorem 4.2]{da2014stochastic}) to \(E(t) \coloneqq \mathop{\mathbb{E}}\qty[\norm{ \hat{\widetilde{X}}_{t} - \hat{X}_{t}}^{2}_{H}]\), we have
  \[
    \dv{E(t)}{t} = 2 \mathop{\mathbb{E}}\qty[\ev{ \hat{\widetilde{X}}_{t} - \hat{X}_{t}, \widetilde{f}(\theta(t), x_0, \hat{\widetilde{X}}_{t}) - f(\theta(t), x_0, \hat{X}_{t}) }_{H} \dot{\theta}(t)].
  \]
  We add \(0 = -\widetilde{f}(\theta(t), x_0, \hat{X}_{t}) + \widetilde{f}(\theta(t), x_0, \hat{X}_{t})\) to the second argument of the inner product to split this into two terms:
  \begin{align}
    \dv{E(t)}{t} &= \overbrace{2 \mathop{\mathbb{E}}\qty[\ev{ \hat{\widetilde{X}}_{t} - \hat{X}_{t}, \widetilde{f}(\theta(t), x_0, \hat{\widetilde{X}}_{t}) -\widetilde{f}(\theta(t), x_0, \hat{X}_{t})}_{H} \dot{\theta}(t)]}^{\text{propagation error term}} \notag \\
    &\phantom{=}\, + \underbrace{2 \mathop{\mathbb{E}}\qty[ \ev{ \hat{\widetilde{X}}_{t} - \hat{X}_{t}, \widetilde{f}(\theta(t), x_0, \hat{X}_{t})  - f(\theta(t), x_0, \hat{X}_{t})}_{H} \dot{\theta}(t)]}_{\text{training error term}}. \label{eqn:b}
  \end{align}
  We place a bound on each term using the Cauchy-Schwarz inequality. For the propagation error term, we make use of the fact that \(\widetilde{\varphi}\) and \(\widetilde{\eta}\) are Lipschitz continuous, so that for each \(t \in [0, \theta^{-1}(\overline{t})]\) and \(x_{0} \in H\), the mapping
  \[
    x \mapsto \widetilde{f}(\theta(t), x_{0}, x) \dot{\theta}(t) = \dot{\theta}(t)\bigl(\widetilde{\varphi}(\theta(t), x_{0}, x) + c(\theta(t)) \widetilde{\eta}(\theta(t), x_{0}, x)\bigr)
  \]
  is Lipschitz continuous in \(H\)-norm with Lipschitz constant \(\widetilde{L}(t)=\mathop{\underset{}{\max}}\qty{1, {c}(\theta(t))} \dot{\theta}(t)L(\theta(t))\). Since by \Cref{lem:tc}, both \(c(\theta(t))\dot\theta(t)\) and \(\dot{\theta}(t)\) are continuous on the compact interval \([0, \theta(\overline{t})]\), the uniform bound
  \begin{align}
      \tilde C(t)=\mathop{\underset{s \in [0, \theta({t})]}{\max}}L(\theta(s))\dot{\theta}(s)(1 + \abs{{c}(\theta(s))})
  \end{align}
  is finite, for any $u\in[0,t]$. It follows that
  \begin{equation}
    2 \mathop{\mathbb{E}}\qty[ \ev{ \hat{\widetilde{X}}_{u} - \hat{X}_{u}, \widetilde{f}(\theta(u), x_0, \hat{\widetilde{X}}_{u}) -\widetilde{f}(\theta(u), x_0, \hat{X}_{u})}_{H} \dot{\theta}(u) ]\leq 2 C(t) E(u). \label{eqn:petb}
  \end{equation}
  For the training error term, we have
  \begin{align}
    &2 \mathop{\mathbb{E}}\qty[\ev{ \hat{\widetilde{X}}_{u} - \hat{X}_{u}, \widetilde{f}(\theta(u), x_0, \hat{X}_{u})  - f(\theta(u), x_0, \hat{X}_{u})}_{H} \dot{\theta}(u)] \notag \\
    &\qquad \leq 2 \mathop{\mathbb{E}}\qty[\norm{ \hat{\widetilde{X}}_{u} - \hat{X}_{u}}_{H} \norm{ \qty(\widetilde{f}(\theta(u), x_0, \hat{X}_{u}) - f(\theta(u), x_0, \hat{X}_{u})) \dot{\theta}(u)}_{H}]\notag \\
    &\qquad \leq E(u) + \mathop{\mathbb{E}}\qty[\norm{ \qty(\widetilde{f}(\theta(u), x_0, \hat{X}_{u}) - f(\theta(u), x_0, \hat{X}_{u})) \dot{\theta}(u)}^{2}_{H}] \notag \\
    &\qquad \leq E(u) + \dot{\theta}(u)^{2}\qty( \mathbb{E}\left[\norm{\tilde\varphi(\theta(u), \hat X_u, x_0)-\varphi(\theta(u), \hat X_u, x_0)}^2_H\right]+
    c^2(\theta(u))\mathbb{E}\left[\norm{\tilde\eta(u, \hat X_u, x_0)-\eta(u,\hat X_u, x_0)}^2_H\right]). \label{eqn:tetb}
  \end{align}
  Substituting our bounds on the propagation error term (Equation \ref{eqn:petb}) and training error term (Equation \ref{eqn:tetb}) into \Cref{eqn:b}, we have
  \[
    \dv{E(u)}{u}  \leq (2\tilde C(t) + 1) E(u) + \dot{\theta}^2(u)\qty( \mathbb{E}\left[\norm{\tilde\varphi(\theta(u), \hat X_u, x_0)-\varphi(\theta(u), \hat X_u, x_0)}^2_H\right]+
    c^2(\theta(u))\mathbb{E}\left[\norm{\tilde\eta(u, \hat X_u, x_0)-\eta(u, \hat X_u, x_0)}^2_H\right]).
  \]
  Let
  \begin{align}
      A(t) &= \mathbb{E}\left[\norm{\tilde\varphi(t, X_t, x_0)-\varphi(t, X_t, x_0)}^2_H\right] \quad\text{and}\quad
      B(t) = \mathbb{E}\left[\norm{\tilde\eta(t, X_t, x_0)-\eta(t, X_t, x_0)}^2_H\right],
  \end{align}
  then applying Grönwall's inequality to \(E(t)\) we have:
  \begin{align}
      E(t) = \mathop{\mathbb{E}}\qty[\norm{ {\widetilde{X}}_{\theta(t)} - {X}_{\theta(t)}}^{2}_{H}] &\leq e^{2\tilde C(t)+1}\int_0^t \dot\theta(u)^2(A(\theta(u))+c^2(\theta(u))B(\theta(u))du.
  \end{align}
  Therefore, the Wasserstein-2 distance between $\tilde\mu_{t|0}(\cdot,x_0)$ and $\mu_{t|0}(\cdot,x_0)$ is upper bounded by
  \begin{align}
      \mathcal{W}_2^2(\tilde\mu_{t|0}, \mu_{t|0})\leq \mathop{\mathbb{E}}\qty[\norm{ {\widetilde{X}}_{t} - {X}_{t}}^{2}_{H}]=E(\theta^{-1}(t))\leq e^{2\tilde C(\theta^{-1}(t))+1} \int_0^{\theta^{-1}(t)} \dot\theta^2(u)(A(\theta(u))+c^2(\theta(u))B(\theta(u))du
      ,
  \end{align}
  Changing variables again, we obtain
  \begin{align}
      \mathcal{W}_2^2(\tilde\mu_{t|0}, \mu_{t|0})\leq e^{C(t)}\int_0^t
      \bigl(A(s)+c^{2}(s)B(s)\bigr)\dot\theta(\theta^{-1}(s))ds
      ,
  \end{align}
  where $C(t)=\max_{s\in[0,t]}2 L(s)\dot\theta(s)(1+|c(\theta(s))|) + 1$. 
  This concludes the proof.
\end{proof}

\section{Instantiation of Framework} \label{sec:dp}
We provide a concrete setup of the framework and algorithms that we use to solve PDE-based forward and inverse problems.

Our theory points to a practical tradeoff in choosing the noise covariance \(C\). The noise must be rough enough for the data \(x_{0}, x_{1}\) to lie in the Cameron-Martin space \(H_{C}\) (see \Cref{rem:hc}), but making it too rough creates a harder learning problem, since \(x_{t}\) becomes less informative and training targets are rougher, and can increase the Lipschitz constant \(\widetilde{L}\) of the learned networks which weakens the error bound in \Cref{thm:w2}. We hypothesise a ``sweet spot'' for noise regularity. Regarding the schedule \(\gamma(t)\), 
\begin{equation}
  \lim\limits_{t \to 0} \frac{t}{\gamma(t)} = \lim\limits_{t \to 1} \frac{1-t}{\gamma(t)} = 0,\label{eqn:integrability}
\end{equation}

\subsection{Hilbert Spaces} \label{sec:hs}
Throughout, we work with data that lie in the Hilbert space of square-integrable functions on a compact Euclidean subset: this will be \(L^{2} = L^{2}(D)\) where \(D = [0, 1]\) for functions defined on a unit interval \(D = [0, 1]^{2}\) for functions on the unit square. We equip this with the canonical \(L^{2}\)-inner product:
\begin{equation}
  \ev{f, g}_{L^{2}} \coloneqq \int_{D} f(x) g(x) \dd{x}.
\end{equation}

We work with two distinct settings for source and target data. For  \textit{homogeneous data}, where \(x_{0}\) and \(x_{1}\) represent similar physical quantities and can be naturally modelled on the same function space (e.g. predicting future fluid vorticity from a past one), we define  \(H \coloneqq L^{2}\). 
For \textit{heterogeneous data}, where \(x_{0}\) and \(x_{1}\) are different physical quantities (e.g. permeability and pressure fields), we use a \textit{product space} \(H = L^{2} \times L^{2}\) for stronger inductive bias. We define new variables  \(x'_{0} = (x_{0}, 0)\) and \(x'_{1} = (0, x_{1})\) and interpolate between the laws of \(x_{0}'\) and \(x_{1}'\), where \(0\) represents the zero function on \(D\).  Our theory holds since the product space \(H\) is still a Hilbert space.


\subsection{Noise}
When \(H = L^{2}\), we define noise \(z\) as samples from a Gaussian process \citep{williams2006gaussian} with zero mean and the radial basis function (RBF) kernel \(k\):
\begin{equation}
  z \sim \mathrm{GP}(0, k), \text{ where } k(x, y) \coloneqq \exp(-\frac{1}{2\ell} \norm{x - y}^{2}_{D}).
\end{equation} This is equivalent to sampling \(z\) from a Gaussian measure \(N(0, C)\) on \(H\) where the covariance operator is trace-class and given by
\begin{equation}
  Cf(x) \coloneqq \int_{D} f(y) k(x, y) \dd{y}.
\end{equation}
We vary length scale \(\ell > 0\) to investigate the impact of noise smoothness on  performance.

In the product space setting where \(H = L^{2} \times L^{2}\), we define the noise \(z\) as a pair of independent samples from this process, i.e. \(z = (z_{0}, z_{1})\) where each component \(z_{0}, z_{1} \overset{\text{i.i.d.}}{\sim} \mathrm{GP}(0, k)\). Formally, this is equivalent to sampling from a GP with matrix-valued kernel \(K\):
\begin{equation}
  z \sim \mathrm{GP}(0, K), \text{ where } K((x_{0}, x_{1}), (y_{0}, y_{1})) \coloneqq \mqty[k(x_{0},y_{0}) & 0 \\ 0 & k(x_{1},y_{1})].
\end{equation}
Henceforth, we continue to use notation pertaining to the homogeneous data setting, but all statements are equally valid for heterogeneous data.

\subsection{Tradeoff Between Noise Roughness and Learnability}\label{sec:dps}

Our theory points to a practical tradeoff in choosing the noise covariance \(C\). The noise must be rough enough for the data \(x_{0}, x_{1}\) to lie in the Cameron-Martin space \(H_{C}\) (see \Cref{rem:hc}), but making it too rough creates a harder learning problem since \(x_{t}\) becomes less informative and training targets are rougher, and can increase the Lipschitz constant \(\widetilde{L}\) of the learned networks which weakens the error bound in \Cref{thm:w2}. This hypothesis of a ``sweet spot'' for noise regularity is confirmed in our experiments below.

The fact that the error bound (Equation \ref{eqn:w2}) increases exponentially in \(\widetilde{L}\) suggests the importance of network regularisation: capacity is optimally matched with the complexity of the data and noise. We employ implicit regularisation below, leaving explicit control over network smoothness to future work.

\subsection{Choice of \texorpdfstring{\(\gamma(t)\)}{γ(t)}} Following \citet{albergo2023stochasticinterpolantsunifyingframework}, we define
\begin{equation}
  \gamma(t) \coloneqq \sqrt{bt(1-t)}.
\end{equation}
This satisfies the conditions  required by \Cref{lem:tc} to permit existence of a suitable regularising time change function \(\theta\). We  provide  conditions on  \(\theta\) for this choice of \(\gamma \) such that the time-changed coefficient \(\hat{c}(t) = c(\theta(t)) \dot{\theta}(t)\) (Equation \ref{eqn:chat}) on the denoiser is finite on \([0, 1]\):

  \begin{restatable}{lemma}{restatelemthetaconditions}\label{lem:thetaconditions}
    Let \(\gamma(t) = \sqrt{bt(1-t)}\). A strictly increasing, bijective, continuously differentiable time change function \(\theta(t)\) on \([0, 1]\) is a valid change-of-time ensuring that \(\hat{c}(t)\) is finite on \([0, 1]\) if and only if \(\theta(t)\) satisfies the following conditions.
    \begin{enumerate}
      \item\label{lem:thetaconditions:1} \(\lim\limits_{t \to 1^{-}} \frac{\dot{\theta}(t)}{2(1-t)} < \infty \); and
      \item\label{lem:thetaconditions:2} \(\lim\limits_{t \to 0^{+}} \frac{\dot{\theta}(t)}{2t} < \infty \) if \(\varepsilon \neq \frac{b}{2}\).
    \end{enumerate}
  \end{restatable}
  \begin{proof}
  We have:
  \begin{equation}
    \hat{c}(t) = \frac{b - 2\varepsilon}{2 \sqrt{b \theta(t)(1-\theta(t))}} \dot{\theta}(t) - \sqrt{\frac{b \theta(t)}{1-\theta(t)}} \dot{\theta}(t) \label{eqn:chatme}
  \end{equation}
  By inspection, \(\hat{c}(t)\)  is finite on all \((0, 1)\) since \(\theta(t) \in (0, 1)\) and \(\dot{\theta}(t)\) is continuous. We analyse the limits at the endpoints by considering each case in turn.

  First, when \(\varepsilon = \frac{b}{2}\), the first term  of \Cref{eqn:chatme} vanishes, reducing analysis to the second term. This is finite on \([0, 1)\), so we need only consider the limit  \(t  \to 1^{-}\). This is finite if and only if \(\lim\limits_{t \to 1^{-}} q(t)\) is finite, where \(q(t) \coloneqq \frac{\dot{\theta}(t)}{\sqrt{1-\theta(t)}}\). Using a substitution \(y(t) = \sqrt{1-\theta(t)}\), we have \(q(t) = -2 \dot{y}(t)\) and hence
  \begin{align}
    \lim\limits_{t \to 1^{-}} q(t) < \infty &\iff \lim\limits_{t \to 1^{-}} -\dot{y}(t) = \lim\limits_{t \to 1^{-}} \frac{y(t)}{1 - t} < \infty\\
    &\iff \lim\limits_{t \to 1^{-}} \frac{y^{2}(t)}{(1-t)^{2}}  = \lim\limits_{t \to 1^{-}} \frac{1-\theta(t)}{(1-t)^{2}} = \lim\limits_{t \to 1^{-}}  \frac{\dot{\theta}(t)}{2(1-t)} < \infty.
  \end{align}
  Therefore, when \(\varepsilon = \frac{b}{2}\), \(\hat{c}(t)\) has a finite continuous extension on \([0, 1]\) if and only if condition (\ref{lem:thetaconditions:1}) in \Cref{lem:thetaconditions} holds.

  Now we consider the case \(\varepsilon \neq \frac{b}{2}\). We now additionally require the first term in \Cref{eqn:chatme} to have finite limits at the endpoints. In the limit \(t \to 1^{-}\), the first term is finite if and only if  \(\lim\limits_{t \to 1^{-}} \frac{\dot{\theta}(t)}{\sqrt{1-\theta(t)}} \) is finite, which is the same condition as discussed above. It remains to check the limit \(t \to 0^{+}\). Here, the first term is finite if and only if \(\lim\limits_{t \to 0^{+}} r(t) < \infty\), where \(r(t) \coloneqq \frac{\dot{\theta}(t)}{\sqrt{\theta(t)}}\). Considering a substitution \(u(t) \coloneqq \sqrt{\theta(t)}\), we have \(r(t) = 2 \dot{u}(t)\) and hence
  \begin{align}
    \lim\limits_{t \to 0^{+}}  r(t) < \infty &\iff \lim\limits_{t \to 0^{+}}  \dot{u}(t) = \lim\limits_{t \to 0^{+}}  \frac{u(t)}{t} < \infty \\
    &\iff \lim\limits_{t \to 0^{+}}  \frac{u^{2}(t)}{t^{2}} = \lim\limits_{t \to 0^{+}} \frac{\theta(t)}{t^{2}} = \lim\limits_{t \to 0^{+}} \frac{\dot{\theta}(t)}{2t} < \infty.
  \end{align}
  Hence, when \(\varepsilon \neq \frac{b}{2}\), \(\hat{c}(t)\) has a finite continuous extension on \([0, 1]\) if and only if both conditions (\ref{lem:thetaconditions:1}) and (\ref{lem:thetaconditions:2}) in \Cref{lem:thetaconditions} hold. This concludes the proof.
  \end{proof}

    Intuitively, conditions (\ref{lem:thetaconditions:1}) and (\ref{lem:thetaconditions:2}) require the \textit{rate of time change} to vanish at least linearly at the endpoints. This controlled deceleration is precisely what resolves the singularity.
    
For SDE inference, we set \(\varepsilon = \frac{b}{2}\), which simplifies the original coefficient to \(c(t) = -\sqrt{\frac{bt}{1-t}}\). This resolves the singularity at \(t=0\) leaving only the singularity at  \(t=1\) to be  managed by the time change. Therefore, only condition (\ref{lem:thetaconditions:1}) is required to ensure the time-changed coefficient \(\hat{c}(t)\) is finite on \([0, 1]\).

\subsection{Choice of \texorpdfstring{\(\alpha(t)\)}{α(t)} and \texorpdfstring{\(\beta(t)\)}{β(t)}}
Following work on rectified flow \citep{liu2022flow}, we choose \(\alpha(t)  \coloneqq 1-t\) and \(\beta(t) \coloneqq t\), which makes the signal \(\alpha(t) x_{0} + \beta(t) x_{1}\) a linear interpolation between source and target data. This straight line path has two advantages:
\begin{enumerate}
  \item The instantaneous velocity is \(\varphi(t, x_{t}) = \mathop{\mathbb{E}}\qty[x_{1} - x_{0} \mid x_{0}, x_{t}]\), which simplifies the training task as the network \(\widetilde{\varphi}\) targets the constant vector \(x_{1} - x_{0}\).
  \item The lack of curvature in the trajectory means that the ODE component is easier to solve during inference, by reducing discretisation error of numerical solvers.
\end{enumerate}

\section{Algorithms}\label{sec:algos}
\subsection{Training and Sampling}
\Cref{alg:training,alg:sampling} outline  concrete implementation details we  use for training and sampling.  During training, time \(t\) is sampled uniformly on \([0, 1]\).  During sampling, we solve the time-changed CB-SDE (\ref{eqn:tccbsde}) using an arbitrary SDE solver (or ODE solver if \(\varepsilon = 0\)) with update function \texttt{OneStep}. In our applications, we employ 2nd-order Euler-Maruyama \citep{sauer2011numerical} which we find to empirically outperform other solvers (see \Cref{sec:solver}). On a computer, functions must be discretised, so we introduce a grid \(\mathcal{G} = \{s^{(j)}\}_{j=1}^{J}\) of sensor locations, where each sensor \(s^{(j)}\) is an input point on the domain \(D\). In calculating the losses, the squared \(H\)-norm becomes a summation over the sensors, leading to standard mean-squared error loss.

    \begin{algorithm}[H]
      \caption{Training}\label{alg:training}
      \begin{algorithmic}[1]%
        \REQUIRE Paired training data \(\mathcal{D} \coloneqq \{(x_{0}^{(i)}, x_{1}^{(i)})\}_{i=1}^{I} \sim \mu\), batch size \(B\), discretisation grid \(\mathcal{G} \coloneqq \{s^{(j)}\}_{j=1}^{J}\) of sensor locations  \(s^{(j)} \in D\)
        \STATE Initialise networks \(\widetilde{\varphi}, \widetilde{\eta}\)
        \WHILE{loss not converged}
        \STATE Sample \(\{z^{(i)}\}_{i=1}^{B}\) evaluated at points in \(\mathcal{G}\), where each \(z^{(i)} \overset{\text{i.i.d.}}{\sim} \operatorname{GP}(0, k)\)
        \STATE Sample  \(\{(x_{0}^{(i)}, x_{1}^{(i)})\}_{i=1}^{B}\) from \(\mathcal{D}\)
        \STATE Sample \(\{t^{(i)}\}_{i=1}^{B} \overset{\text{i.i.d.}}{\sim} \operatorname{U}([0, 1])\)
        \STATE Let interpolant \(x_{t}^{(i)} \leftarrow \alpha(t^{(i)}) x_{0}^{(i)} +\beta(t^{(i)})x_{1}^{(i)} + \gamma(t^{(i)}) z^{(i)}\)
        \STATE Let loss \(\mathcal{L}(\widetilde{\varphi}) \leftarrow \frac{1}{B}\sum_{i=1}^{B}\frac{1}{J} \sum_{j=1}^{J} \norm{\big[\widetilde{\varphi}(t^{(i)}, x^{(i)}_{t}) - \big(\dot{\alpha}(t^{(i)}) x_{0}^{(i)} + \dot{\beta}(t^{(i)}) x_{1}^{(i)}\big)\big](s^{(j)})}^{2} \)
        \STATE Let loss \(
          \mathcal{L}(\widetilde{\eta}) \leftarrow \frac{1}{B} \sum_{i=1}^{B} \frac{1}{J}\sum_{j=1}^{J} \norm{\big[\widetilde{\eta}(t^{(i)}, x_{t}^{(i)}) - z^{(i)}\big](s^{(j)})}^{2},
        \)
        \STATE Perform gradient step on \(\mathcal{L}(\widetilde{\varphi})\) and \(\mathcal{L}(\widetilde{\eta})\).
        \ENDWHILE
      \end{algorithmic}
    \end{algorithm}
    \begin{algorithm}[H]
      \setlength{\abovedisplayskip}{2pt}
      \setlength{\belowdisplayskip}{2pt}
      \caption{Sampling}
      \label{alg:sampling}
      \begin{algorithmic}[1]
        \REQUIRE Test dataset \(\mathcal{T} \coloneqq \{x_{0}^{(i)}\}_{i=1}^{I} \sim \mu_{0}\), number of time steps \(T \geq 1\), trained networks \(\widetilde{\varphi}, \widetilde{\eta}\), parameter \(\varepsilon \geq 0\), time change function \(\theta\), any SDE (or ODE if \(\varepsilon = 0\)) solver with update function \texttt{OneStep}
        \STATE Construct time-changed approximate drift \begin{equation}\hat{\widetilde{f}}(t, x_{0}, x) \leftarrow \widetilde{\varphi}(\theta(t), x_{0}, x)\dot{\theta}(t) + \qty(\dot{\gamma}(\theta(t)) - \frac{\varepsilon}{\gamma(\theta(t))}) \widetilde{\eta}(\theta(t), x_{0}, x)\dot{\theta}(t)\end{equation}
        \STATE Let \(\Delta t \leftarrow \frac{1}{T}\)
        \STATE Let \(X_{0}^{(i)} \leftarrow x_{0}^{(i)}\)
        \WHILE{ \(t \neq 1\) }
        \STATE \(t \leftarrow 0\)
        \STATE \(X_{t + \Delta t}^{(i)} \leftarrow \texttt{OneStep}(t, x_{0}^{(i)}, X_{t}^{(i)}, \hat{\widetilde{f}}, \varepsilon, C, \Delta t)\)
        \STATE \(t \leftarrow t + \Delta t\)
        \ENDWHILE
      \OUTPUT \(\{X_{1}^{(i)}\}_{i=1}^{I}\)
      \end{algorithmic}
    \end{algorithm}

The algorithms are stated for the forward problem, i.e. bridging from source \(x_{0} \sim \mu_{0}\) to conditional target \(\mu_{1 \mid 0}(\dd{x_{1}, x_{0}})\), but analogous results for the reverse problem are given by considering approximations \(\widetilde{\varphi}^{\text{rev}}, \widetilde{\eta}^{\text{rev}}\) as in \Cref{sec:backwards} and initialising \(X_{0}^{(i)} \leftarrow x_{1}^{(i)}\) during sampling (note that the process is also solved forward in time).

An alternative to direct sensor-based discretisation is to project data onto an eigenbasis of \(C\) (a Karhunen-Loeve expansion; \citealp{stark1986probability}). This is akin to \citet{phillips2022spectral}, who consider (finite-dimensional) diffusion models in the spectral domain. However, we avoid this as our neural operator architectures are designed to learn their own optimal spectral representations \citep{li2020fourier,rahman2022u}. A fixed pre-projection would create an unnecessary bottleneck and require a costly eigenvalue problem.


\subsection{Numerical Integration}\label{sec:solver}
\Cref{tbl:solver} compares 1st-and 2nd-order Euler-Maruyama \citep{sauer2011numerical} against the 2nd-order Heun-type solver proposed by \citet{karras2022elucidating} for our 2D datasets. 2nd-order Euler-Maruyama demonstrates superior performance. The underperformance of Heun is reflective of its particular specialism for DMs \citep{karras2022elucidating}, which does not necessarily translate into superior performance for SIs.
\begin{table}[H]
  \centering
  \sisetup{detect-weight, mode=text}
  \setlength{\aboverulesep}{0pt}
  \setlength{\belowrulesep}{0pt}
  \setlength{\heavyrulewidth}{1.5pt}
  \setlength{\lightrulewidth}{0.4pt}
  \setlength{\cmidrulewidth}{0.4pt}
  \renewcommand{\arraystretch}{1.1}
  \begin{tabular}{lS[table-format=1.1] S[table-format=2.1] S[table-format=1.1] S[table-format=2.1]}
    \toprule
    \multirow{2}{*}{\bfseries{Solver}} &  \multicolumn{2}{c}{\bfseries{Darcy Flow (\%)}} & \multicolumn{2}{c}{\bfseries{Navier-Stokes (\%)}} \\

    \cmidrule(r){2-3} \cmidrule(r){4-5}
    & {Forward} & {Inverse} & {Forward} & {Inverse}  \\
    \midrule
    1st-order EM + ODE & 2.0 & 2.8 &  1.2 &  \bfseries 4.6 \\
    1st-order EM + SDE  & 2.3 & 3.4 & 1.6 & 6.3  \\
    \midrule
    2nd-order EM + ODE  & \bfseries 1.9 & \bfseries 2.7 &  \bfseries 1.0 &  \bfseries 4.6 \\
    2nd-order EM + SDE  & 2.3 & 3.7 & 1.4 & 6.1 \\
    \midrule
    2nd-order Heun + ODE  & 2.1 & 2.9 &  1.2 &  \bfseries 4.6 \\
    2nd-order Heun + SDE  & 2.3 & 6.3 & 1.5 & 6.2  \\
    \bottomrule\\[-0.5em]
  \end{tabular}
  \caption{Comparison of different solvers for our proposed infinite-dimension SI framework on 2D forward and inverse tasks. Metrics are in relative \(L^{2}\)-error, except for the inverse task of Darcy flow which is the binary error rate. Best values are \textbf{bold}.}\label{tbl:solver}
\end{table}

\section{Experimental Details}\label{sec:exp-details}
\subsection{Datasets}
For completeness, we provide a detailed overview of the datasets we used in our experiments. 
\subsubsection{1D Darcy Flow} \label{app:1d-darcy}
We follow \citet{ingebrand2025basis} and consider a non-linear variant of Darcy's equation in 1D \citep{whitaker1986flow}. Darcy flow describes fluid flow along a porous medium:
\begin{equation}
  -\dv{w}\qty[\qty(0.2 + \psi^{2}(w)) \dv{\psi(w)}{w}] = u(w), \qquad w \in [0, 1], \psi(0) = \psi(1) = 0,\label{eqn:d1d}
\end{equation}
where \(u \sim \operatorname{GP}\) with RBF kernel and length scale \(0.05\) is a \textit{forcing function} and the solution \(\psi \in L^{2}\) is \textit{pressure}. Intuitively,  \(u(w)\) describes the environmental factors of fluid inflow/outflow at distances \(w \in [0, 1]\) along a unit-length pipe, and \(\psi(w)\) is the corresponding pressure along this pipe at \(w\). The quantity \(0.2 + \psi^{2}(w)\) is a pressure-dependent permeability.

The forward problem is as follows: given a specific forcing function \(u\), we aim to predict the (unique) corresponding pressure solution \(\psi\). Similarly, for the inverse problem, we are given a pressure solution \(\psi\) and aim to predict the forcing function \(u\) which produced that solution.

Using the code from \citet{ingebrand2025basis}, we generate a  dataset of paired data with a 9000/1000 train/evaluation split, discretised on 128 evenly-spaced points on \([0, 1]\).

Computationally, we calculate these norms using Riemann summation over the 128 evenly spaced gridpoints on \([0, 1]\).

\subsubsection{2D PDE Problems} \label{sec:2dsetup}
We use a 49000/1000/1000 train/dev/evaluation split of paired functions discretised on \(64\!\times\!64\) evenly-spaced points on the unit square. 
\paragraph{Darcy Flow} We consider static flow through a porous 2D medium, governed by
\begin{equation}
- \grad \cdot \qty(a(w) \grad \psi(w)) = 1, \qquad w \in (0,1)^{2},
\end{equation}
where analogously to \Cref{eqn:d1d}, the scalar-valued pressure field \(\psi(w)\) is zero at the boundaries. The forcing function is identically \(1\), and \(a\nobreak\sim\nobreak{h}_{\sharp}\nobreak\operatorname{N}(0, \qty(- \Delta + 9 I)^{-2})\) is a binary-valued permeability field. Concretely, we sample according to the Gaussian measure \(\operatorname{N}(0, \qty(-\Delta + 9I)^{-2})\) on the Hilbert space \(L^{2}\), where \(\Delta\) is a Laplacian operator with zero Neumann boundary condition \citep{cheng2005heritage} and \(I\) is the identity operator. Then to obtain \(a\), we push these samples through  \(h\nobreak:\nobreak\mathbb{R}\nobreak\to\nobreak\qty{3,12}\) which returns \(12\) for positive inputs and \(3\) otherwise. Hence,  \(a(w)\) is either \(3\) or \(12\) for any \(w \in D\).

The dataset is formed of pairs \((a, \psi)\): the forward task is to predict pressure solution \(\psi\) given permeability \(a\), while the inverse task is to predict permeability given pressure solution. Since permeability and pressure are distinct physical phenomena, we employ the form of interpolants for heterogeneous data (\(H = L^{2} \times L^{2}\)). 

\paragraph{Navier-Stokes} We additionally consider the Navier-Stokes equations \citep{batchelor2000introduction} describing a scalar-valued vorticity field \(\upsilon(w, \tau) = \grad \times v(w, \tau)\) on the 2D unit square, where \(v(w, \tau)\) is vector-valued fluid velocity at point \(w \in D\) and time  \(\tau \in [0, 1]\), governed by:
\begin{align}
\pdv{\tau} \upsilon(w, \tau) + v(w, \tau) \cdot  \grad \upsilon(w, \tau) &= 0.001 \Delta \upsilon(w, \tau) + u(w),&& \qquad w \in (0, 1)^{2}, t \in (0, 1], \\
\grad \cdot v(w, t) &= 0,&& \qquad w \in (0, 1)^{2}, t \in [0, 1],\\
\upsilon(w, 0) &= a(w), &&\qquad w \in (0, 1)^{2},
\end{align}
where \(0.001\) is the viscosity. The forcing function \(u(w)\) is defined as
\begin{equation}u(w_{1}, w_{2}) \coloneqq 0.1 \sin 2\pi (w_{1} + w_{2}) + 0.1 \cos 2 \pi (w_{1} + w_{2}),\end{equation}
and  initial condition \(a(w)\) is sampled from the Gaussian measure \(\operatorname{N}(0, 7^{\frac{3}{2}}(-\Delta + 49 I)^{-\frac{5}{2}})\) with periodic boundary conditions.

The dataset is formed of pairs \((\upsilon_{0}, \upsilon_{1})\): the forward task is to predict final vorticity \(\upsilon_{1} \coloneqq \upsilon(\cdot, 1)\) given initial vorticity \(\upsilon_{0} \coloneqq \upsilon(\cdot, 0)\), while the inverse task is to predict the initial condition given final vorticity. Since both tasks  predict vorticity, we employ the form of interpolants for homogeneous data (\(H = L^{2}\)), so that both target and source reside in a single channel.

Functions are normalised to have zero mean and unit standard deviation across the dataset. Permeability \(a\) is normalised to take binary values \(+1\) and \(-1\).

\subsection{Hyperparameters} \label{sec:2instantiation}
\Cref{tbl:d2hyp} details the hyperparameters we used for experiments on 2D PDEs. The U-NO architecture hyperparameters are from \citet{yao2025guideddiffusionsamplingfunction}:
this architecture balances high spectral capacity with strong regularisation. Notably, each spectral convolution uses 32 Fourier modes, which for \(64\!\times\!64\)-resolution data gives the operator maximum capacity to represent high-frequency details. To regularise this, the spectral convolution weights are represented via a low-rank tensor decomposition. This reduces the number of parameters by nearly an order of magnitude, constraining network capacity in line with  design guidance in \Cref{sec:dps}.

\begin{table}[H]
  \centering
  \sisetup{detect-weight, mode=text}
  \renewcommand{\arraystretch}{1}
  \begin{tabular}{lc}
    \toprule
    \textbf{Hyperparameter} & \textbf{Value} \\
    \midrule
    Batch size & 96 \\
    Learning rate & \(0.0001\) \\
    Learning rate warmup (steps) & \(17000\) \\
    Training steps & \(9.8\) million \\
    EMA half-life (steps) & 4000 \\
    \midrule
    \(\alpha(t)\) in \(x_{t}\) & \(1-t\) \\
    \(\beta(t)\) in \(x_{t}\) & \(t\) \\
    \(\gamma(t)\) in \(x_{t}\) & \(\sqrt{0.01 t(1-t)}\) \\
    \midrule
    Neural operator & U-NO \citep{rahman2022u} \\
    FFT Modes & 32 \\
    Blocks & 4 \\
    Width multipliers & \([1, 2, 4, 4]\) \\
    Attention resolutions & [8] \\
    Positional embeddings & Yes \\
    Rank & \(0.01\) \\
    Dropout & \(0.13\) \\
    \bottomrule\\[-0.5em]
  \end{tabular}
  \caption{Hyperparameters used in the velocity and denoising networks for Darcy flow and Navier-Stokes on the unit square }\label{tbl:d2hyp}
\end{table}

\section{Additional Figures} \label{app:figures}
In \Cref{fig:2devolution}, we present qualitative results for the 2D Darcy flow and Navier–Stokes problems described earlier. For each task, we show the temporal evolution of predictions under ODE and SDE dynamics.
\begin{figure}\label{fig:2devolution}
\begin{center}
\begin{subfigure}{0.9\linewidth}
  \centering
  \includegraphics[width=\linewidth]{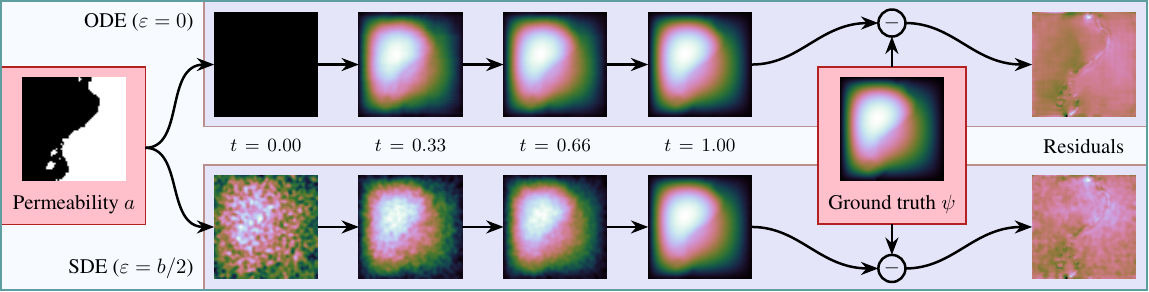}
  \caption{Darcy flow (forward)}\label{fig:2devolution:a}
\end{subfigure}

\vspace{1em}

\begin{subfigure}{0.9\linewidth}
  \centering
  \includegraphics[width=\linewidth]{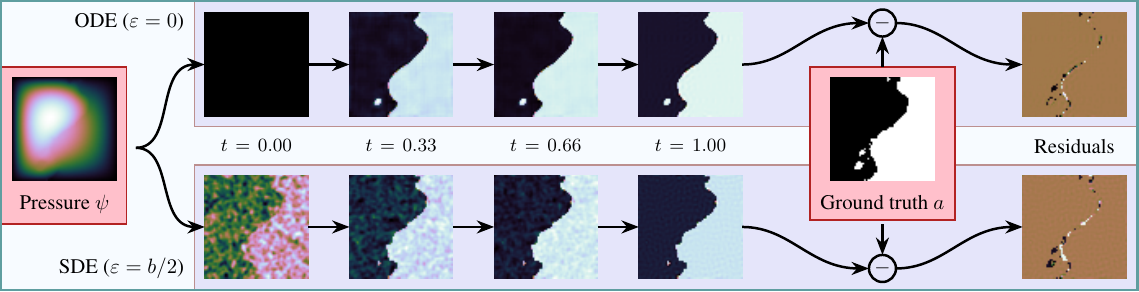}
  \caption{Darcy flow (inverse)}\label{fig:2devolution:b}
\end{subfigure}

\vspace{1em}

\begin{subfigure}{0.9\linewidth}
  \centering
  \includegraphics[width=\linewidth]{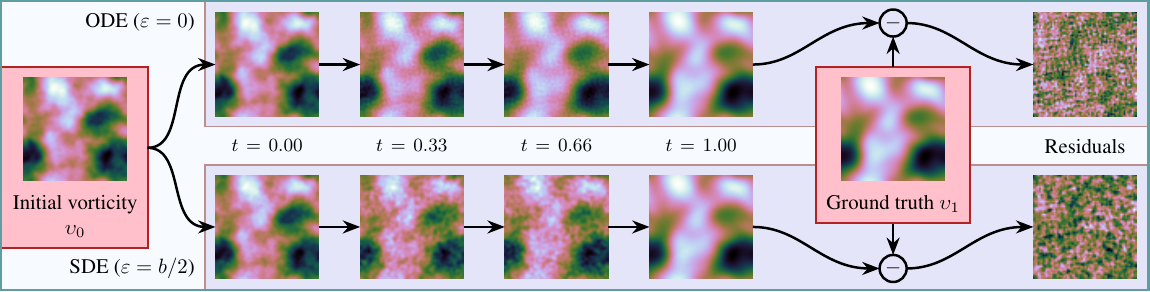}
  \caption{Navier-Stokes (forward)}
\end{subfigure}

\vspace{1em}

\begin{subfigure}{0.9\linewidth}
  \centering
  \includegraphics[width=\linewidth]{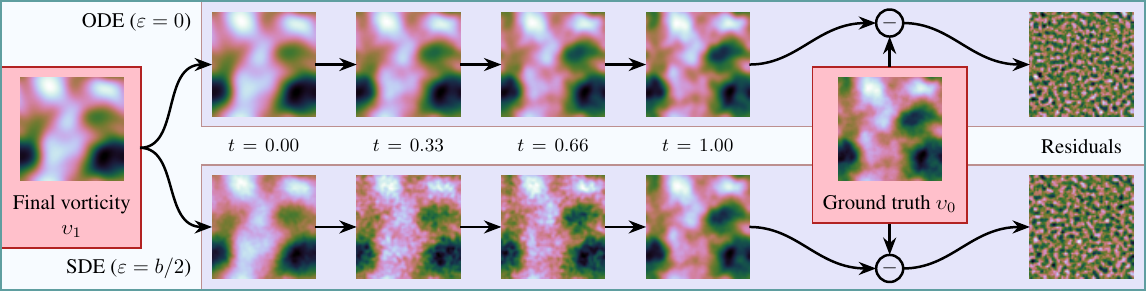}
  \caption{Navier-Stokes (inverse)}
\end{subfigure}
\caption{Evolution of predictions from the ODE (top) and a single sample of the SDE (bottom), for forward and inverse tasks. We randomly select an example from the test set for Darcy flow and Navier-Stokes. We set the RBF length scale \(\ell = 0.02\).} \label{fig:2devolution}
\end{center}
\end{figure}
\end{document}